\documentclass{article}

\usepackage{amssymb}
\usepackage{amsmath}
\usepackage{amsthm}
\usepackage{enumitem}
\usepackage{booktabs}
\usepackage{natbib}
\usepackage{lineno}
\usepackage[curve]{xypic}
\usepackage{tikz}
\usepackage{authblk} 
\usetikzlibrary{matrix,arrows}
\renewcommand{\geq}{\geqslant}

\renewcommand{\phi}{\varphi}

\newcommand{\argmax}{\operatornamewithlimits{argmax}}
\newcommand{\prof}[1]{{\boldsymbol{#1}}}
\newcommand{\N}{\mathcal{N}}

\newcommand{\imp}{\rightarrow}
\newcommand{\val}{\text{\it Val}}
\newcommand{\Edges}{V\!\times\! V}

\newcommand{\CONFIG}{P[S^+\!,S^-]}
\newcommand{\Fmaj}{F_{\text{\it maj}}}

\newtheorem{definition}{Definition}
\newtheorem{example}{Example}
\newtheorem{theorem}{Theorem}
\newtheorem{proposition}[theorem]{Proposition}
\newtheorem{lemma}[theorem]{Lemma}
\newtheorem{corollary}[theorem]{Corollary}
\newtheorem{fact}[theorem]{Fact}

\renewcommand{\cite}{\citep}
\newcommand{\thmcite}[1]{\citeauthor{#1}, \citeyear{#1}}

\begin{document}

%%%%%%%%%%%%%%%%%%%%%%%%%%%%%%%%%%%%%%%%%%%%%%%%%%%%%%%%%%%%%%%%%%%%%%%%%%%%%%%%

\title{Graph Aggregation\thanks{This work refines and extends papers presented at COMSOC-2012~\cite{EndrissGrandiCOMSOC2012} and ECAI-2014~\cite{EndrissGrandiECAI2014}. We are grateful for the extensive feedback received from several anonymous reviewers as well as from audiences at the  SSEAC Workshop on Social Choice and Social Software held in Kiel in 2012, the Dagstuhl Seminar on Computation and Incentives in Social Choice in 2012, the KNAW Academy Colloquium on Dependence Logic held at the Royal Netherlands Academy of Arts and Sciences in Amsterdam in 2014, a course on logical frameworks for multiagent aggregation given at the 26th European Summer School in Logic, Language and Information (ESSLLI-2014) in T\"ubingen in 2014, the Lorentz Center Workshop on Clusters, Games and Axioms held in Leiden in 2015, and lectures delivered at Sun Yat-Sen University in Guangzhou in 2014 as well as at \'Ecole Normale Sup\'erieure and Pierre \& Marie Curie University in Paris in 2016. This work was partly supported by COST Action IC1205 on Computational Social Choice. It was completed while the first author was hosted at the University of Toulouse in 2015 as well as Paris-Dauphine University, Pierre \& Marie Curie University, and the London School of Economics in~2016.}}

\author[1]{Ulle Endriss}
%\ead{ulle.endriss@uva.nl}

\author[2]{Umberto Grandi}
%\ead{umberto.grandi@ut-capitole.fr}

\affil[1]{ILLC, University of Amsterdam, The Netherlands}
\affil[2]{IRIT, University of Toulouse, France}
\date{}

\maketitle

\begin{abstract}
Graph aggregation is the process of computing a single output graph that constitutes a good compromise between several input graphs, each provided by a different source. One needs to perform graph aggregation in a wide variety of situations, e.g., when applying a voting rule (graphs as preference orders), when consolidating conflicting views regarding the relationships between arguments in a debate (graphs as abstract argumentation frameworks), or when computing a consensus between several alternative clusterings of a given dataset (graphs as equivalence relations). In this paper, we introduce a formal framework for graph aggregation grounded in social choice theory. Our focus is on understanding which properties shared by the individual input graphs will transfer to the output graph returned by a given aggregation rule. We consider both common properties of graphs, such as transitivity and reflexivity, and arbitrary properties expressible in certain fragments of modal logic. Our results establish several connections between the types of properties preserved under aggregation and the choice-theoretic axioms satisfied by the rules used. The most important of these results is a powerful impossibility theorem that generalises Arrow's seminal result for the aggregation of preference orders to a large collection of different types of graphs.
\end{abstract}

%\begin{keyword}
%social choice theory \sep
%collective rationality \sep
%impossibility theorems \sep
%graph theory \sep
%modal logic \sep
%preference aggregation \sep
%belief merging \sep
%consensus clustering \sep
%argumentation theory 
%\end{keyword}

%%%%%%%%%%%%%%%%%%%%%%%%%%%%%%%%%%%%%%%%%%%%%%%%%%%%%%%%%%%%%%%%%%%%%%%%%%%%%%%%
\section{Introduction}	   
%%%%%%%%%%%%%%%%%%%%%%%%%%%%%%%%%%%%%%%%%%%%%%%%%%%%%%%%%%%%%%%%%%%%%%%%%%%%%%%%

% WHAT IS GRAPH AGGREGATION?
\noindent
Suppose each of the members of a group of autonomous agents provides us with a different directed graph that is defined on a common set of vertices. Graph aggregation is the task of computing a single graph over the same set of vertices that, in some sense, represents a good compromise between the various individual views expressed by the agents. 
%In this paper, we introduce the problem of graph aggregation, we define a general framework for studying different methods of aggregation, we prove a number of results that highlight both opportunities and limitations in graph aggregation, and we demonstrate its relevance by instantiating our general results to a wide range of concrete application scenarios. 
%These application scenarios include preference aggregation, belief merging, consensus clustering, and argumentation in multiagent systems.
%
% WHY IS THIS IMPORTANT AND RELEVANT TO AI? / EXAMPLES FOR APPLICATIONS
Graphs are ubiquitous in computer science and artificial intelligence (AI). For example, in the context of decision support systems, an edge from vertex~$x$ to vertex~$y$ might indicate that alternative~$x$ is \emph{preferred} to alternative~$y$. In the context of modelling interactions taking place on an online debating platform, an edge from $x$ to $y$ might indicate that argument~$x$ undercuts or otherwise \emph{attacks} argument~$y$. And in the context of social network analysis, an edge from $x$ to $y$ might express that person~$x$ is \emph{influenced} by person~$y$. 
How to best perform graph aggregation is a relevant question in these three domains, as well as in any other domain where graphs are used as a modelling tool and where particular graphs may be supplied by different agents or originate from different sources. For example, in an election, i.e., in a group decision making context, we have to aggregate the preferences of several voters. In a debate, we sometimes have to aggregate the views of the individual participants in the debate. And when trying to understand the dynamics within a community, we sometimes have to aggregate information coming from several different social networks. %---that is, in this last example the agents are the providers of the different social networks.

% APPROACH: SCT 
In this paper, we introduce a formal framework for studying graph aggregation in general abstract terms and we discuss in detail how this general framework can be instantiated to specific application scenarios. We introduce a number of concrete methods for performing aggregation, but more importantly, our framework provides tools for evaluating what constitutes a ``good'' method of aggregation and it allows us to ask questions regarding the existence of methods that meet a certain set of requirements. Our approach is inspired by work in social choice theory~\cite{ArrowEtAlHBSCW2002}, which offers a rich framework for the study of aggregation rules for preferences---a very specific class of graphs. In particular, we adopt the \emph{axiomatic method} used in social choice theory, as well as other parts of economic theory, to identify intuitively desirable properties of aggregation methods, to define them in mathematically precise terms, and to systematically explore their logical consequences. 

% CENTRAL CONCEPT: COLLECTIVE RATIONALITY / DIFFERENT WAYS OF EXPRESSING PROPERTIES
An aggregation rule maps any given \emph{profile} of graphs, one for each agent, into a single graph, which we will often refer to as the \emph{collective graph}. The central concept we focus on in this paper is the \emph{collective rationality} of aggregation rules with respect to certain properties of graphs. Suppose we consider an agent rational only if the graph she provides has certain properties, such as being reflexive or transitive. Then we say that a given aggregation rule~$F$ is collectively rational with respect to that property of interest if and only if $F$ can guarantee that that property is preserved during aggregation. For example, if we aggregate individual graphs by computing their \emph{union} (i.e., if we include an edge from $x$ to $y$ in our collective graph if at least one of the individual graphs includes that edge), then it is easy to see that the property of \emph{reflexivity} will always transfer. On the other hand, the property of \emph{transitivity} will not always transfer. For example, if we aggregate two graphs over the set of vertices $V=\{x,y,z\}$, one consisting only of the edge $(x,y)$ and one consisting only of the edge $(y,z)$, then although each of these two graphs is (vacuously) transitive, their union is not, as it is missing the edge $(x,z)$. Thus, the union rule is collectively rational with respect to reflexivity, but not with respect to transitivity. We study collective rationality with respect to some such well-known and widely used properties of graphs, but also with respect to large families of graph properties that satisfy certain \emph{meta-properties}. We explore both a semantic and a syntactic approach to defining such meta-properties. In our semantic approach, we identify certain high-level features of graph properties that determine the kind of aggregation rules that are collectively rational with respect to them. For example, transitivity is what we call an ``implicative'' property: under certain circumstances (namely in the presence of an edge from $x$ to $y$), the inclusion of an edge from $y$ to $z$ \emph{implies} the inclusion of an edge from $x$ to $z$. In our syntactic approach, we consider graph properties that can be expressed in particular syntactic fragments of a logical language. To this end, we make use of the language of modal logic~\cite{BlackburnEtAl2001}. This allows us to establish links between the syntactic properties of the language used to express the integrity constraints we would like to see preserved during aggregation and the axiomatic properties of the rules used. 

%Moreover, a compact languages is needed to express constraints. Model checking is widespread and many modal languages it is quite easy [[check, examples]].

% PREVIEW OF RESULTS
We prove both \emph{possibility} and \emph{impossibility results}. A possibility result establishes that every aggregation rule belonging to a certain class of rules (typically defined in terms of certain axioms) is collectively rational with respect to all graph properties that satisfy a certain meta-property. An impossibility result, on the other hand, establishes that it is impossible to define an aggregation rule belonging to a certain class that would be collectively rational with respect to any graph property that meets a certain meta-property---or that the only such aggregation rules would be clearly very unattractive for other reasons. Our main result is such an impossibility theorem. It is a generalisation of Arrow's seminal result for preference aggregation~\cite{Arrow1963}. Arrow's Theorem says that no aggregation rule that satisfies the \emph{Pareto principle} and that is \emph{independent of irrelevant alternatives} can be guaranteed to always preserve the transitivity and completeness of the preference orders being aggregated. Here, the Pareto principle stipulates that the aggregation rule should respect any unanimously held strict preferences over pairs of alternatives (thus, if all individuals rank $x$ above $y$, so should the rule). The independence property expresses that it should be possible to decide on the relative rankings of alternatives in a pair-by-pair fashion (thus, to decide whether $x$ should be ranked above $y$, the rule should only have to consider the relative rankings of $x$ and $y$ provided by the individuals, rather than, say, how they rank $x$ and $z$). The only exceptions admitted by Arrow's Theorem are rules for fewer than three alternatives and rules that are \emph{dictatorial}, in the sense of always returning the preference order of one specific agent (the dictator). Our approach of working with meta-properties has two advantages. First, it permits us to give conceptually simple proofs for powerful results with a high degree of generality. Second, it makes it easy to instantiate our general results to obtain specific results for specific application scenarios. For example, Arrow's Theorem follows immediately from our more general result by checking that the properties of graphs that represent preference orders satisfy the meta-properties featuring in our theorem, yet our proof of the general theorem is arguably simpler than a direct proof of Arrow's Theorem. This is so, because the meta-properties we use very explicitly exhibit specific features required for the proof, while those features are somewhat hidden in the specific properties of transitivity and completeness. Similarly, we show how alternative instantiations of our general result easily generate both known and new results in other domains, such as the aggregation of plausibility orders (which has applications in nonmonotonic reasoning and belief merging) and the aggregation of equivalence relations (which has applications in clustering analysis).

% RELATED WORK / APPLICATIONS
Our work builds on and is related to contributions in the field of social choice theory, starting with the seminal contribution of \citet{Arrow1963}. This concerns, in particular, contributions to the theory of voting and preference aggregation~\cite{FishburnJET1970,KirmanSondermannJET1972,HanssonPC1976,Sen1986,PiniEtAlJLC2009,ArrowEtAlHBSCW2002}, but also judgment aggregation~\cite{ListPettitEP2002,GardenforsEP2006,DietrichListJTP2007,DokowHolzmanJET2010,HerzbergEckertMSS2012,GrandiEndrissAIJ2013,ListPuppe2009}. In fact, in terms of levels of generality, graph aggregation may be regarded as occupying the middle ground between preference aggregation (most specific) and judgment aggregation (most general). In computer science, these frameworks are studied in the field of computational social choice~\cite{BrandtEtAlHBCOMSOC2016}. As we shall discuss in some detail, graph aggregation is an abstraction of several more specific forms of aggregation taking place in a wide range of different domains. Preference aggregation is but one example. Aggregation of specific types of graphs has been studied, for instance, in nonmonotonic reasoning~\cite{DoyleWellmanAIJ1991}, belief merging~\cite{MaynardZhangLehmannJAIR2003}, social network analysis~\cite{WhiteEtAlAJS1976}, clustering~\cite{FishburnRubinsteinJC1986}, and argumentation in multiagent systems~\cite{TohmeEtAlFoIKS2008}. As we shall see, several of the results obtained in these earlier contributions are simple corollaries of our general results on graph aggregation. 

% PAPER OVERVIEW
The remainder of this paper is organised as follows. In Section~\ref{sec:aggregation}, we introduce our framework for graph aggregation. This includes the discussion of several application scenarios, the definition of a number of concrete aggregation rules, and the formulation of various axioms identifying intuitively desirable properties of such rules. It also includes the definition of the concept of collective rationality. Finally, we prove a number of basic results in Section~\ref{sec:aggregation}: characterisation results linking rules and axioms, as well as possibility results linking axioms and collective rationality requirements.
In Section~\ref{sec:impossibility}, we present our impossibility results for graph aggregation rules that are collectively rational with respect to graph properties meeting certain semantically defined meta-properties. There are two such results. One identifies conditions under which the only available rules are so-called \emph{oligarchies}, under which the outcome is always the intersection of the graphs provided by a subset of the agents (the oligarchs). A second result shows that, under slightly stronger assumptions, the only available rules are the \emph{dictatorships}, where a single agent completely determines the outcome for every possible profile. Much of Section~\ref{sec:impossibility} is devoted to the definition and illustration of the meta-properties featuring in these results. Once they are in place, the proofs are relatively simple.  
In Section~\ref{sec:modal}, we introduce our approach to describing collective rationality requirements in syntactic terms, using the language of modal logic. Our results in Section~\ref{sec:modal} establish simple conditions on the syntax of the specification of a graph property that are sufficient for guaranteeing that the property in question will be preserved under aggregation.  The grounding of our approach in modal logic also allows us to provide a deeper analysis of the concept of collective rationality by considering the preservation of properties at three different levels, corresponding to the three levels naturally defined by the notions of Kripke frame, Kripke model, and possible world, respectively.  
In Section~\ref{sec:applications}, we discuss four of our application scenarios in more detail, focusing on application scenarios previously discussed in the AI literature. We show how our general results allow us to derive new simple proofs of known results, how they clarify the status of some of these results, and how they allow us to obtain new results in these domains of application.
Section~\ref{sec:conclusion}, finally, concludes with a brief summary of our results and pointers to possible directions for future work.

%%%%%%%%%%%%%%%%%%%%%%%%%%%%%%%%%%%%%%%%%%%%%%%%%%%%%%%%%%%%%%%%%%%%%%%%%%%%%%%%
\section{Graph Aggregation}\label{sec:aggregation}
%%%%%%%%%%%%%%%%%%%%%%%%%%%%%%%%%%%%%%%%%%%%%%%%%%%%%%%%%%%%%%%%%%%%%%%%%%%%%%%%

\noindent
In this section, we introduce a simple framework for graph aggregation. The basic definitions are given in Section~\ref{sec:definitions}. While this is a general framework that is independent of specific application scenarios and specific choices regarding the aggregation rule used, we briefly discuss several such specific scenarios in Section~\ref{sec:examples} and suggest definitions for several specific aggregation rules in Section~\ref{sec:rules}. We then approach the analysis of aggregation rules from two different but complementary angles. First, in Section~\ref{sec:axioms}, we define several \emph{axiomatic properties} of aggregation rules that a user may wish to impose as requirements when looking for a ``fair'' or ``well-behaved'' aggregation rule for a specific application. We also prove a number of simple results that show how some of these axioms relate to each other and to some of the aggregation rules defined earlier. Second, in Section~\ref{sec:CR}, we introduce the central concept of \emph{collective rationality} and we prove a number of simple positive results that show how enforcing certain axioms allows us to guarantee collective rationality with respect to certain graph properties.

%%%%%%%%%%%%%%%%%%%%%%%%%%%%%%%%%%%%%%%%%%%%%%%%%%%%%%%%%%%%%%%%%%%%%%%%%%%%%%%%
\subsection{Basic Notation and Terminology}\label{sec:definitions}
%%%%%%%%%%%%%%%%%%%%%%%%%%%%%%%%%%%%%%%%%%%%%%%%%%%%%%%%%%%%%%%%%%%%%%%%%%%%%%%%

\noindent
Fix a finite set of \emph{vertices} $V\!$. A (directed) \emph{graph} $G=\langle V,E\rangle$ based on $V$ is defined by a set of \emph{edges} $E\subseteq\Edges$.
We write $xEy$ for $(x,y)\in E$.
As $V$ is fixed, $G$ is in fact fully determined by $E$. We therefore identify sets of edges $E\subseteq\Edges$ with the graphs $G=\langle V,E\rangle$ they define. 
For any kind of set~$S$, we use $2^S$ to denote the powerset of~$S$.
So $2^{\Edges}$ is the set of all graphs.
We use $E(x):=\{y\in V \mid (x,y)\in E\}$ to denote the set of \emph{successors} of a vertex~$x$ in a set of edges~$E$ and $E^{-1}(y) := \{x\in V \mid (x,y)\in E\}$ to denote the set of \emph{predecessors} of~$y$ in~$E$.

\begin{table}[t]
\centering
\begin{tabular}{ll} \toprule
\textsc{Property} & \textsc{First-Order Condition} \\ \midrule
Reflexivity & $\forall x.xEx$ \\
Irreflexivity & $\neg\exists x.xEx$ \\
Symmetry & $\forall xy.(xEy \imp yEx)$ \\
Antisymmetry & $\forall xy.(xEy\wedge yEx\imp x=y)$ \\
Right Euclidean & $\forall xyz.[(xEy\wedge xEz) \imp yEz]$ \\
Left Euclidean  & $\forall xyz.[(xEy\wedge zEy) \imp zEx]$ \\
Transitivity & $\forall xyz.[(xEy\wedge yEz) \imp xEz]$ \\
Negative Transitivity & $\forall xyz.[xEy \imp (xEz\vee zEy)]$ \\
Connectedness & $\forall xyz.[(xEy\wedge xEz) \imp (yEz \vee zEy)]$ \\
Completeness & $\forall xy.[x\not=y\imp (xEy \vee yEx)]$ \\
%Strong Completeness & $\forall xy.(xEy \vee yEx)$ \\
%Functionality & $\forall xyz.(xEy\wedge xEz\imp y=z)$ \\
%Church-Rosser & $\forall xy.[xEy\wedge xEz \imp \exists w.(yEw\wedge zEw)]$ \\
Nontriviality & $\exists xy.xEy$ \\
Seriality & $\forall x.\exists y.xEy$ \\
%Acyclicity & $\neg\exists x_1 x_2\cdots x_k.[x_1 E x_2 \wedge \cdots \wedge x_{k-1}E x_k \wedge x_k E x_1]$ \\ % not first-order!
\bottomrule 
\end{tabular}
\caption{Common properties of directed graphs.\label{tab:graph-properties}}
\end{table}

A given graph may or may not satisfy a specific \emph{property}, such as transitivity or reflexivity. Table~\ref{tab:graph-properties} recalls the definitions of several such properties.\footnote{Some of these may be less well known than others, so let us briefly review the less familiar definitions. The two Euclidean properties encode Euclid's idea that ``things which equal the same thing also equal one another''. Negative transitivity, a property commonly assumed in the economics literature on preferences, may equivalently be expressed as $\forall xyz.[(\neg xEy\wedge\neg yEz)\imp\neg xEz]$, which explains the name of the property. Completeness requires any two distinct vertices to be related one way or the other. Connectedness only requires two (not necessarily distinct) vertices to be related one way or the other if they are both reachable from some common predecessor (the term ``connectedness'' is commonly used in the modal logic literature~\cite{BlackburnEtAl2001}). Nontriviality excludes the empty graph, while seriality (also a term used in the modal logic literature) requires every vertex to have at least one successor.}
We will often be interested in families of graphs that all satisfy several of these properties. For instance, a \emph{weak order} is a directed graph that is reflexive, transitive, and complete.
It will often be useful to think of a graph property $P$, such as transitivity, as a subset of $2^{\Edges}$ (the set of all graphs over the set of vertices~$V$).
For two disjoint sets of edges $S^+$ and $S^-$ and a graph property $P\subseteq 2^{\Edges}$, let $\CONFIG = \{E\in P\mid S^+\subseteq E\ \mbox{and}\ S^-\cap E=\emptyset\}$ denote the set of graphs in $P$ that include all of the edges in $S^+$ and none of those in~$S^-$.

%\uenote{What is the correct terminology for what we call ``connected'' in graph theory? Our terminology seems common in modal logic, but not in graph theory.}
%\uenote{Strong completeness could be removed.}
%\uenote{Acyclicity is mentioned informally in the text, but not defined. Probably should indeed be omitted at this point, as it is not first-order definable and has none of our three meta-properties.}

Let $\N=\{1,\ldots,n\}$ be a finite set of (two or more) \emph{individuals} (or \emph{agents}). We will often refer to subsets of $\N$ as \emph{coalitions} of individuals.
Suppose every individual $i\in\N$ specifies a graph $E_i\subseteq\Edges\!$. This gives rise to a \emph{profile} $\prof{E}=(E_1,\ldots,E_n)$.
We use $N^\prof{E}_e:=\{i\in\N \mid e\in E_i\}$ to denote the coalition of individuals accepting edge~$e$ under profile~$\prof{E}$. 

\begin{definition}
An \textbf{aggregation rule} is a function $F:(2^{\Edges})^n\to2^{\Edges}$, mapping any given profile of individual graphs into a single graph.
\end{definition}

\noindent
We will sometimes denote the outcome $F(\prof{E})$ obtained when applying an aggregation rule~$F$ to a profile $\prof{E}$ simply as $E$ and refer to it as the  \emph{collective graph}. An example for an aggregation rule is the \emph{majority rule}, accepting a given edge if and only if more than half of the individuals accept it. More examples will be provided in Section~\ref{sec:rules}.

%%%%%%%%%%%%%%%%%%%%%%%%%%%%%%%%%%%%%%%%%%%%%%%%%%%%%%%%%%%%%%%%%%%%%%%%%%%%%%%%
\subsection{Examples of Application Scenarios}\label{sec:examples}
%%%%%%%%%%%%%%%%%%%%%%%%%%%%%%%%%%%%%%%%%%%%%%%%%%%%%%%%%%%%%%%%%%%%%%%%%%%%%%%%

\noindent
Directed graphs are ubiquitous in computer science and beyond. They have been used as modelling devices for a wide range of applications. We now sketch a number of different application scenarios for graph aggregation, each requiring different types of graphs (satisfying different properties) to model relevant objects of interest, and each requiring different types of aggregation rules.

\begin{example}[Preferences]
Our main example for a graph aggregation problem will be preference aggregation as classically studied in social choice theory~\cite{Arrow1963}. In this context, vertices are interpreted as alternatives available in an election and the graphs considered are weak orders on these alternatives, interpreted as preference orders. Our aggregation rules then reduce to so-called social welfare functions. Social welfare functions, which return a preference order for every profile of individual preference orders, are similar objects as voting rules, which only return a winning alternative for every profile. While the types of preferences typically considered in classical social choice theory are required to be complete, recent work in AI has also addressed the aggregation of partial preference orders~\cite{PiniEtAlJLC2009}, corresponding to a larger family of graphs than the weak orders. In the context of aggregating complex preferences defined over combinatorial domains, graph aggregation can also be used to decide which preferential dependencies between different variables one should try to respect, based on the dependencies reported by the individual decision makers~\cite{AiriauEtAlIJCAI2011}.
\end{example}

\begin{example}[Knowledge]
If we think of $V$ as a set of possible worlds, then a graph on $V$ that is reflexive and transitive (and possibly also symmetric) can be used to model an agent's knowledge: $(x,y)$ being an edge means that, if $x$ is the actual world, then our agent will consider $y$ a possible world~\cite{Hintikka1962}. 
%If we have several individuals each providing a model of our agent's knowledge, then we might want to aggregate that information into a consensus model of the agent's knowledge.\footnote{That is, our individuals are not situated ``within'' such an epistemic model, but rather they each provide us with a full such model.}
If we aggregate the graphs of several agents by taking their intersection, then the resulting collective graph represents the distributed knowledge of the group, i.e., the knowledge the members of the group can infer by pooling all their individual resources. If, on the other hand, we aggregate by taking the union of the individual graphs, then we obtain what is sometimes called the shared or mutual knowledge of the individual agents, i.e., the part of the knowledge available to each and every individual on their own. Finally, if we aggregate by computing the transitive closure of the union of the individual graphs, then we obtain a model of the group's common knowledge~\cite[p.~512]{Egre2011}. These concepts play a role in disciplines as diverse as epistemology~\cite{LewisAJP1996}, game theory~\cite{AumannAS1976}, and distributed systems~\cite{HalpernMosesJACM1990}.
\end{example}

\begin{example}[Nonmonotonic reasoning]
When an intelligent agent attempts to update her beliefs or to decide what action to take, she may resort to several patterns of common-sense inference that will sometimes be in conflict with each other. To take a famous example, we may wish to infer that Nixon is a pacifist, because he is a Quaker and Quakers by default are pacifists, and we may at the same time wish to infer that Nixon is not a pacifist, because he is a Republican and Republicans by default are not pacifists. In a popular approach to nonmonotonic reasoning in AI, such default inference rules are modelled as graphs that encode the relative plausibility of different conclusions~\cite{Shoham1987}. Thus, here the possible conclusions are the vertices and we obtain a graph by linking one vertex with another, if the former is considered at least as plausible as the latter. Conflict resolution between different rules of inference then requires us to aggregate such plausibility orders, to be able to determine what the ultimately most plausible state of the world might be~\cite{DoyleWellmanAIJ1991}.
\end{example}

\begin{example}[Social networks]
We may also think of each of the graphs in a profile as a different social network relating members of the same population. One of these networks might describe work relations, another might model family relations, and a third might have been induced from similarities in online purchasing behaviour. Social networks are often modelled using undirected graphs, which we can simulate in our framework by requiring all graphs to be symmetric. Aggregating individual graphs then amounts to finding a single meta-network that describes relationships at a global level. Alternatively, we may wish to aggregate several graphs representing snapshots of the same social network at different points in time.  The meta-network obtained can be helpful when studying the social structures within the population under scrutiny~\cite{WhiteEtAlAJS1976}. 
\end{example}

%\uenote{They really seem to do some kind of aggregation in \cite{WhiteEtAlAJS1976}, but I did not check the details. An additional reference supporting the notion that social network aggregation is not a silly idea would be good. UG: I did not find anything where aggregation methods are explicitely used. There is a wikipedia page on this: \url{https://en.wikipedia.org/wiki/Social_network_aggregation}, but it's too general and not exactly on the topic.}

\begin{example}[Clustering]
Clustering is the attempt of partitioning a given set of data points into several clusters. The intention is that the data points in the same cluster should be more similar to each other than each of them is to data points belonging to one of the other clusters. This is useful in many disciplines, including information retrieval and molecular biology, to name but two examples. However, the field is lacking a precise definition of what constitutes a ``correct'' partitioning of the data and there are many different clustering algorithms, such as $k$-means or single-linkage clustering, and even more parameterisations of those basic algorithms~\cite{TanEtAl2005}. Observe that every partitioning that might get returned by a clustering algorithm induces an equivalence relation (i.e., a graph that is reflexive, symmetric, and transitive): two data points are equivalent if and only if they belong to the same cluster. Finding a compromise between the solutions suggested by several clustering algorithms is what is known as consensus clustering~\cite{GionisEtAlTKDD2007}. This thus amounts to aggregating several graphs that are equivalence relations.
\end{example}

\begin{example}[Argumentation]
In a so-called abstract argumentation framework, arguments are taken to be vertices in a graph and attacks between arguments are modelled as directed edges between them~\cite{DungAIJ1995}. A graph property of interest in this context is acyclicity, as that makes it easier to decide which arguments to ultimately accept. If we think of $V$ as the collection of arguments proposed in a debate, a profile $\prof{E}=(E_1,\dots, E_n)$ specifies an attack relation for each of a number of agents that we may wish to aggregate into a collective attack relation before attempting to determine which of the arguments might be acceptable to the group. Recent work has addressed the challenge of aggregating several abstract argumentation frameworks from a number of angles, e.g., by proposing concrete aggregation methods grounded in work on belief merging \cite{CosteMarquisEtAlAIJ2007}, by investigating the computational complexity of aggregation~\cite{DunneEtAlCOMMA2012}, and by analysing what kinds of profiles we may reasonably expect to encounter in this context~\cite{AiriauEtAlAAMAS2016}.
\end{example}

\begin{example}[Logic]
Graph aggregation is also at the core of recent work on the aggregation of different logics~\cite{WenLiuLORI2013}. The central idea here is that every logic is defined by a consequence relation between formulas. Thus, given a set of formulas, we can think of a logic~$\mathcal{L}$ as the graph corresponding to the consequence relation defining~$\mathcal{L}$. Aggregating several such graphs then gives rise to a new logic.
%In related work, \citet{WenLiuLORI2013} introduce a problem they call \emph{logic aggregation:} Several agents each have their own consequence relation over formulas (defining a logic) and we try to find a good compromise, i.e., a collective consequence relation. 
Thus, this is an instance of our graph aggregation problem, except that for the case of logic aggregation it is more natural to model the set of vertices as being infinite.\footnote{All results reported in this paper remain true if we permit graphs with infinite sets~$V$ of vertices. However, for ease of exposition and as most applications are more naturally modelled using finite graphs, we do not explore this generalisation here. The finiteness of the set~$\N$ of agents, however, is crucial. It will be exploited in the proofs of Lemmas~\ref{lem:oligarchy-filter} and~\ref{lem:dictator-ultrafilter} below, on which all of our theorems rely.} 
%\uenote{Actually, I'm not completely sure the point about finiteness of $V$ not mattering is correct ... Maybe in the neutrality lemma we need finiteness to show that we can really traverse the entire set of pairs of vertices? We could also just omit this footnote.}
\end{example}

\noindent
Recall that we have assumed that every individual specifies a graph on \emph{the same} set of vertices~$V$. This is a natural assumption to make in all of our examples above, but in general we might also be interested in aggregating graphs defined on different sets of vertices. For instance, Coste-Marquis et al.~\cite{CosteMarquisEtAlAIJ2007} have argued that, in the context of merging argumentation frameworks, the case of agents who are not all aware of the exact same set of arguments is of great practical interest. Observe that also in this case our framework is applicable, as we may think of $V$ as the union of all the individual sets of vertices (with each individual only providing edges involving ``her'' vertices).

We will return to several of these application scenarios in greater detail in Section~\ref{sec:applications}.

%%%%%%%%%%%%%%%%%%%%%%%%%%%%%%%%%%%%%%%%%%%%%%%%%%%%%%%%%%%%%%%%%%%%%%%%%%%%%%%%
\subsection{Aggregation Rules}\label{sec:rules}
%%%%%%%%%%%%%%%%%%%%%%%%%%%%%%%%%%%%%%%%%%%%%%%%%%%%%%%%%%%%%%%%%%%%%%%%%%%%%%%%

\noindent
Next, we define a number of concrete aggregation rules. We begin with three that are particularly simple, the first of which we have already introduced informally.

\begin{definition}
The (strict) \textbf{majority rule} is the aggregation $\Fmaj$ with $\Fmaj : \prof{E} \mapsto \{e\in\Edges : |N^{\prof{E}}_e| > \frac{n}{2}\}$. 
\end{definition}

\begin{definition}
The \textbf{intersection rule} is the aggregation rule $F_\cap$ with $F_\cap : \prof{E} \mapsto E_1\cap\cdots\cap E_n$.
\end{definition}

\begin{definition}
The \textbf{union rule} is the aggregation rule $F_\cup$ with $F_\cup : \prof{E} \mapsto E_1\cup\cdots\cup E_n$.
\end{definition}

\noindent
In related contexts, the intersection rule is also known as the \emph{unanimity rule}, as it requires unanimous approval from all individuals for an edge to be accepted. Similarly, the union rule is a \emph{nomination rule}, as nomination by just one individual is enough for an edge to get accepted.

Under a \emph{quota rule}, an edge will be included in the collective graph if the number of individuals accepting it meets a certain quota. A \emph{uniform} quota rule uses the same quota for every edge.

\begin{definition}%[Quota rules]
A \textbf{quota rule} is an aggregation rule $F_q$ defined via a function $q : \Edges \to\{0,1,\ldots,n{+}1\}$, associating each edge with a quota, by stipulating $F_q : \prof{E} \mapsto \{e\in\Edges : |N^{\prof{E}}_e| \geq q(e)\}$. $F_q$~is called \textbf{uniform} in case $q$ is a constant function.
\end{definition}

\noindent
The class of uniform quota rules includes the three simple rules we have seen earlier as special cases:
the (strict) \emph{majority rule}~$\Fmaj$ is the uniform quota rule with $q=\lceil\frac{n+1}{2}\rceil$,
%The \emph{weak majority rule} is the uniform quota rule with $q=\lceil\frac{n}{2}\rceil$. 
%Observe that for an odd number $n$ of agents, the weak and the strict majority rule coincide.
the \emph{intersection rule}~$F_\cap$ is the uniform quota rule with $q=n$, and
the \emph{union rule}~$F_\cup$ is the uniform quota rule with $q=1$.
We call the uniform quota rules with $q=0$ and $q=n{+}1$ the \emph{trivial} quota rules; $q=0$ means that \emph{all} edges will be included in the collective graph and $q=n{+}1$ means that \emph{no} edge will be included (independently of the profile encountered).
The idea of using quota rules is natural and widespread. For example, quota rules have also been studied in judgment aggregation~\cite{DietrichListJTP2007}. 

We now introduce a new class of aggregation rules specifically designed for graphs that is inspired by approval voting~\cite{BramsFishburn2007}. Imagine we associate with each vertex an election in which all the possible successors of that vertex are the candidates (and in which there may be more than one winner). Each individual votes by stating which vertices they consider acceptable successors.

\begin{definition}%[Successor-approval rules]
A \textbf{successor-approval rule} is an aggregation rule~$F_v$ defined via a function $v:(2^V)^n\to 2^V$, associating each profile of approval sets of vertices with a winning set of vertices, by stipulating $F : \prof{E} \mapsto \{(x,y)\in\Edges \mid y\in v(E_1(x),\ldots,E_n(x))\}$. 
\end{definition}

\noindent
We call $v$ the \emph{choice function} associated with $F_v$. It takes a vector of sets of vertices, one for each agent, and returns another such set. We will only be interested in choice functions~$v$ that are $(i)$~\emph{anonymous} and $(ii)$~\emph{neutral}, i.e., for which $(i)$~$v(S_1,\ldots,S_n)=v(S_{\pi(1)},\ldots,S_{\pi(n)})$ for any permutation $\pi:\N\to\N$ and for which $(ii)$~$\{i\in\N \mid x\in S_i\} = \{i\in\N \mid y\in S_i\}$ entails $x\in v(S_1,\ldots,S_n) \Leftrightarrow y\in v(S_1,\ldots,S_n)$. 
There are a number of natural choices for~$v$. For example, if we use the classical approval voting rule $v:(S_1,\ldots,S_n) \mapsto \argmax_{x\in V}|\{i\in\N : x\in S_i\}|$, we end up accepting for each vertex those outgoing edges that have maximal support.
Alternatively, we might want to accept all edges receiving above-average support. While classical approval voting will typically result in very ``sparse'' output graphs, intuitively the latter rule will return graphs that have similar attributes as the input graphs.
A third option is to use ``even-and-equal'' cumulative voting with $v : (S_1,\ldots,S_n) \mapsto \argmax_{x\in V} \sum_{i|x\in S_i}\frac{1}{|S_i|}$, i.e., to let each individual distribute her weight evenly over the successors she approves of. This would be attractive, for instance, under an epistemic interpretation, where agents specifying fewer edges might be considered more certain about those edges.\footnote{Note that in case no individual graph has any outgoing edges for~$x$, the successor-approval rules defined in terms of the $\argmax$-operator would accept \emph{all} edges emanating from $x$. This will usually not be desirable and in practice has to be taken care of by defining a suitable exception.}
Finally, observe that the uniform quota rules (but not the general quota rules) are a special case of the successor-approval rules. We obtain $F_q$ with the constant function $q:e\mapsto k$, mapping any given edge to the fixed quota~$k$, by using $v : (S_1,\ldots,S_n) \mapsto \{x\in V : |\{i\in\N : x\in S_i\}| \geq k\}$.

While we will not do so in this paper, it is also possible to adapt the \emph{distance-based rules}---familiar from preference aggregation, belief merging, and judgment aggregation~\citep{Kemeny1959,KoniecznyPinoPerezJLC2002,MillerOshersonSCW2009}---to the case of graph aggregation. Such rules select a collective graph that satisfies certain properties and that minimises the distance to the individual graphs (for a suitable notion of distance and a suitable form of aggregating such distances). A downside of this approach is that distance-based rules are typically computationally intractable~\citep{HemaspaandraEtAlTCS2005,EndrissEtAlJAIR2012,LangSlavkovikECAI2014,EndrissDeHaanAAMAS2015}, while quota and successor-approval rules have very low complexity. 

We can also adapt the \emph{representative-voter rules}~\citep{EndrissGrandiAAAI2014} to the case of graph aggregation. Here, the idea is to return one of the input graphs as the output, and for every profile to pick the input graph that in some sense is ``most representative'' of the views of the group. 

\begin{definition}\label{def:representative-voter-rule}
A \textbf{representative-voter rule} is an aggregation rule~$F$ that is such that for every profile~$\prof{E}$ there exists an individual $i^\star\in\N$ such that $F(\prof{E})=E_{i^\star}$.
\end{definition}

\noindent
For instance, we might pick the input graph that is closest to the outcome of the majority rule. This majority-based representative-voter rule also has very low complexity. While we will not study any specific representative-voter rule in this paper, in Section~\ref{sec:world-CR} we will briefly discuss this class of rules as a whole.

%While of great practical importance, we shall not consider distance-based or representative-voter rules any further in this paper, because they all ensure that the collective graph satisfies the graph properties we desire ``by design''. Thus, the main question we will investigate in this paper, namely the question into the circumstances under which desirable graph properties are preserved under aggregation simply does not arise for these rules.  

We conclude our presentation of concrete (families of) aggregation rules with a number of rules that, intuitively speaking, are not very attractive.

\begin{definition}
The \textbf{dictatorship} of individual $i^\star\in\N$ is the aggregation rule $F_{i^\star}$ with $F_{i^\star} : \prof{E} \mapsto E_{i^\star}$.
\end{definition}

\noindent
Thus, for any given profile of input graphs, $F_{i^\star}$ always simply returns the graph submitted by the dictator~$i^\star$. Note that every dictatorship is a representative-voter rule, but the converse is not true. 

\begin{definition}
The \textbf{oligarchy} of coalition $C^\star\subseteq\N$, with $C^\star$ being nonempty, is the aggregation rule $F_{C^\star}$ with $F_{C^\star} : \prof{E} \mapsto \displaystyle\bigcap_{i\in C^\star} E_i$.
\end{definition}

\noindent
Thus, $F_{C^\star}$ always returns the intersection of the graphs submitted by the oligarchs in the coalition~$C^\star$. So an individual in $C^\star$ can veto the acceptance of any given edge, but she cannot enforce its acceptance. In case $C^\star$ is a singleton, we obtain a dictatorship. In case $C^\star=\N$, we obtain the intersection rule. 
%In case $C^\star$ is the empty set, we obtain a constant rule that always returns the complete graph~$\Edges$.

%%%%%%%%%%%%%%%%%%%%%%%%%%%%%%%%%%%%%%%%%%%%%%%%%%%%%%%%%%%%%%%%%%%%%%%%%%%%%%%%
\subsection{Axiomatic Properties and Basic Characterisation Results}
\label{sec:axioms}
%%%%%%%%%%%%%%%%%%%%%%%%%%%%%%%%%%%%%%%%%%%%%%%%%%%%%%%%%%%%%%%%%%%%%%%%%%%%%%%%

\noindent
When choosing an aggregation rule, we need to consider its properties. In social choice theory, such properties are called \emph{axioms}~\cite{Sen1986}. We now introduce several basic axioms for graph aggregation.
The first such axiom is an independence condition that requires that the decision of whether or not a given edge~$e$ should be part of the collective graph should only depend on which of the individual graphs include~$e$. 
This corresponds to well-known axioms in preference and judgment aggregation~\citep{Arrow1963,ListPuppe2009}.

\begin{definition}%[IIE]
An aggregation rule $F$ is called independent of irrelevant edges (\textbf{IIE}) if $N^{\prof{E}}_e = N^{\prof{E'}}_e$ implies $e\in F(\prof{E}) \Leftrightarrow e\in F(\prof{E'})$. % for any edge $e\in V^2$ and profiles $\prof{G},\prof{G'}\in\G^\N$.
\end{definition}

\noindent
That is, if exactly the same individuals accept $e$ under profiles $\prof{E}$ and $\prof{E'}$, then $e$ should be part of either both or none of the corresponding collective graphs. 
The definition above applies to \emph{all} edges $e\in\Edges$ and \emph{all} pairs of profiles $\prof{E},\prof{E'}\in (2^{\Edges})^n$. For the sake of readability, we shall leave this kind of universal quantification implicit also in later definitions.

IEE is a desirable property, because---if it can be satisfied---it greatly simplifies aggregation, in both computational and conceptual terms. As we shall see, some of the arguably most natural aggregation rules, the quota rules defined earlier, satisfy IIE. At the same time, as we shall also see, IEE is a very demanding property that is hard to satisfy if we are interested in richer forms of aggregation. Indeed, IIE will turn out to be at the very centre of our impossibility results. 

While very much a standard axiom, we might be dissatisfied with IIE for not making reference to the fact that edges are defined in terms of vertices. Our next two axioms are much more graph-specific and do not have close analogues in preference or judgment aggregation. The first of them requires that the decision of whether or not to collectively accept a given edge $e=(x,y)$ should only depend on which edges with the same source~$x$ are accepted by the individuals. That is, acceptance of an edge may be influenced by what agents think about other edges, but not those edges that are sufficiently unrelated to the edge under consideration. 
Below we write $F(\prof{E})(x)$ for the set of successors of vertex~$x$ in the set of edges in the collective graph~$F(\prof{E})$, and similarly $F(\prof{E})^{-1}(y)$ for the predecessors of $y$ in $F(\prof{E})$.

\begin{definition}%[IIS]
An aggregation $F$ is called independent of irrelevant sources (\textbf{IIS}) if $E_i(x)=E'_i(x)$ for all individuals $i\in\N$ implies
$F(\prof{E})(x) = F(\prof{E'})(x)$. % for any vertex~$x\in V$ and profiles $\prof{G}, \prof{G'}\in \G^\N$.
\end{definition}

%Similarly, the next axiom requires that collective acceptance of an edge $e=(x,y)$ should only depend on the pattern of individual acceptance for those edges with the same target~$y$.

\begin{definition}%[IIT]
An aggregation rule $F$ is called independent of irrelevant targets (\textbf{IIT}) if $E_i^{-1}(y)={E'_i}^{-1}(y)$ for all individuals $i\in\N$ implies $F(\prof{E})^{-1}(y) = F(\prof{E'})^{-1}(y)$. % for any vertex~$y\in V$ and profiles $\prof{G}, \prof{G'}\in \G^\N$.
\end{definition}

\noindent
Both IIS and IIT are strictly weaker than IIE. That is, we obtain the following result, which is easy to verify (simple counterexamples can be devised to show that the converse does not hold):
%The precise relative strength of our independence axioms is illustrated by the following simple result, which is easy to verify. 

\begin{proposition}\label{prop:IIE}
If an aggregation rule is IIE, then it is also both IIS and IIT. 
\end{proposition}

\noindent
The fundamental economic principle of \emph{unanimity} requires that an edge should be accepted by a group in case all individuals in that group accept it.

\begin{definition}%[Unanimity]
An aggregation rule $F$ is called \textbf{unanimous} if it is always the case that $F(\prof{E})\supseteq E_1\cap\cdots\cap E_n$. 
\end{definition}

\noindent
A requirement that, in some sense, is dual to unanimity is to ask that the collective graph should only include edges that are part of at least one of the individual graphs. In the context of ontology aggregation this axiom has been introduced under the name \emph{groundedness}~\cite{PorelloEndrissJLC2014}.

\begin{definition}%[Groundedness]
An aggregation $F$ is called \textbf{grounded} if it is always the case that $F(\prof{E})\subseteq E_1\cup\cdots\cup E_n$. 
\end{definition}

\noindent
The next three axioms are standard desiderata and closely modelled on their counterparts in judgment aggregation~\cite{ListPuppe2009}. The first two of them, anonymity and neutrality, are basic symmetry requirements with respect to individuals and edges, respectively.

\begin{definition}%[Anonymity]
An aggregation rule $F$ is called \textbf{anonymous} if $F(E_1,\ldots,E_n) = F(E_{\pi(1)},\ldots,E_{\pi(n)})$ for any permutation $\pi:\N\to\N$.
\end{definition}

\begin{definition}%[Neutrality]
An aggregation rule $F$ is called \textbf{neutral} if $N^{\prof{E}}_e = N^{\prof{E}}_{e'}$ implies $e\in F(\prof{E}) \Leftrightarrow e'\in F(\prof{E})$.
\end{definition}

\begin{definition}
An aggregation rule $F$ is called \textbf{monotonic} if $E'_i\setminus E_i \subseteq \{e\}$ for all individuals $i\in\N$ implies $e\in F(\prof{E}) \Rightarrow e\in F(\prof{E'})$.
\end{definition}

\noindent
Observe that $E'_i\setminus E_i \subseteq \{e\}$ for all $i\in\N$ means that the profiles $\prof{E}$ and $\prof{E'}$ are identical, except that some individuals who do not accept edge~$e$ in the former profile do accept it in the latter.

The link between aggregation rules and axiomatic properties is expressed in so-called characterisation results. For each rule (or class of rules), the aim is to find a set of axioms that uniquely define this rule (or class of rules, respectively).
A simple adaptation of a result by \citet{DietrichListJTP2007} yields the following characterisation of the class of quota rules:

\begin{proposition}\label{prop:quota}
An aggregation rule is a quota rule if and only if it is anonymous, monotonic, and~IIE.
\end{proposition}

\begin{proof}
To prove the left-to-right direction we simply have to verify that the quota rules all have these three properties. For the right-to-left direction, observe that, to accept a given edge $(x,y)$ in the collective graph, an IIE aggregation rule will only look at the set of individuals~$i$ such that $xE_i y$. If the rule is also anonymous, then the acceptance decision is based only on the number of individuals accepting the edge. Finally, by monotonicity, there will be some minimal number of individual acceptances required to trigger collective acceptance. That number is the quota associated with the edge under consideration.
\end{proof}

\noindent
If we add the axiom of neutrality, then we obtain the class of uniform quota rules. If we furthermore impose unanimity and groundedness, then this excludes the trivial quota rules.
Similarly, it is easy to verify that IIS essentially characterises the class of successor-approval rules:

%\ugnote{We are not putting any characterisation of the majority rule. I am totally ok with this, we would need another axiom.}
%\uenote{Agreed. For both PA and JA the characterisation of the majority rule actually works via collective rationality. This wouldn't work here, but we'd need a May-style acceptance-rejection neutrality axiom. We'd also have to think carefully about how to deal with the even-number-of-agents case. So let's not get into this.}

\begin{proposition}
An aggregation rule is a successor-approval rule (with an anonymous and neutral choice function) if and only if it is anonymous, neutral, and IIS.
\end{proposition}

\noindent
An extreme form of violating anonymity is to use a \emph{dictatorial} or an \emph{oligarchic} aggregation rule, i.e., a rule that is either a dictatorship or an oligarchy (unless the oligarchy in question is the full set~$\N$). 

Sometimes we will only be interested in the properties of an aggregation rule as far as the \emph{nonreflexive} edges $e=(x,y)$ with $x\not=y$ are concerned. Specifically, we call $F$ \emph{neutral on nonreflexive edges} (or just \emph{NR-neutral}) if $N^{\prof{E}}_{(x,y)}=N^{\prof{E}}_{(x',y')}$ implies $(x,y)\in F(\prof{E}) \Leftrightarrow (x',y')\in F(\prof{E})$ for all $x\not=y$ and $x'\not=y'$. Analogously, we call $F$ \emph{dictatorial on nonreflexive edges} (or \emph{NR-dictatorial}) if there exists an individual $i^\star\in\N$ such that $(x,y)\in F(\prof{E}) \Leftrightarrow (x,y)\in E_{i^\star}$ for all $x\not=y$. Finally, we call $F$ \emph{oligarchic on nonreflexive edges} (or \emph{NR-oligarchic}) if there exists a nonempty coalition $C^\star\subseteq\N$ such that $(x,y)\in F(\prof{E}) \Leftrightarrow (x,y)\in \bigcap_{i\in C^\star} E_i$ for all $x\not=y$.

% That is, NR-neutrality is slightly weaker than neutrality and NR-nondictatoriality is slightly stronger than nondictatoriality. [UE: obvious now that we only work with positive properties]

%\uenote{Plan is to avoid the terms ``nondictatorial'' and ``nonoligarchic'' and to express everything positively. UG:ok}

%%%%%%%%%%%%%%%%%%%%%%%%%%%%%%%%%%%%%%%%%%%%%%%%%%%%%%%%%%%%%%%%%%%%%%%%%%%%%%%%
\subsection{Collective Rationality and Basic Possibility Results}\label{sec:CR}
%%%%%%%%%%%%%%%%%%%%%%%%%%%%%%%%%%%%%%%%%%%%%%%%%%%%%%%%%%%%%%%%%%%%%%%%%%%%%%%%

\noindent
To what extent can a given aggregation rule ensure that a given property that is satisfied by each of the individual input graphs will be preserved during aggregation? This question relates to a well-studied concept in social choice theory, often referred to as \emph{collective rationality}~\cite{Arrow1963,ListPettitEP2002}. In the literature, collective rationality is usually defined with respect to a specific property that should be preserved (e.g., the transitivity of preferences or the logical consistency of judgments). Here, instead, we formulate a definition that is parametric with respect to a given graph property.\footnote{In previous work on binary aggregation, a variant of judgment aggregation, we have used the term collective rationality in the same sense, with the property to be preserved under aggregation being encoded in the form of an integrity constraint~\cite{GrandiEndrissAIJ2013}.}

\begin{definition}\label{def:CR}
An aggregation rule $F$ is called \textbf{collectively rational} with respect to a graph property~$P$ if $F(\prof{E})$ satisfies $P$ whenever all of the individual graphs in $\prof{E}=(E_1,\ldots,E_n)$ do.
\end{definition}

\noindent
To illustrate the concept, let us consider two examples. Both concern the majority rule, but different graph properties. The first is a purely abstract example, while the second has a natural interpretation of graphs as preference relations.

\begin{example}[Collective rationality]
\label{ex:CR}
Suppose three individuals provide us with three graphs over the same set $V=\{x,y,z,w\}$ of four vertices, as shown to the left of the dashed line below: 

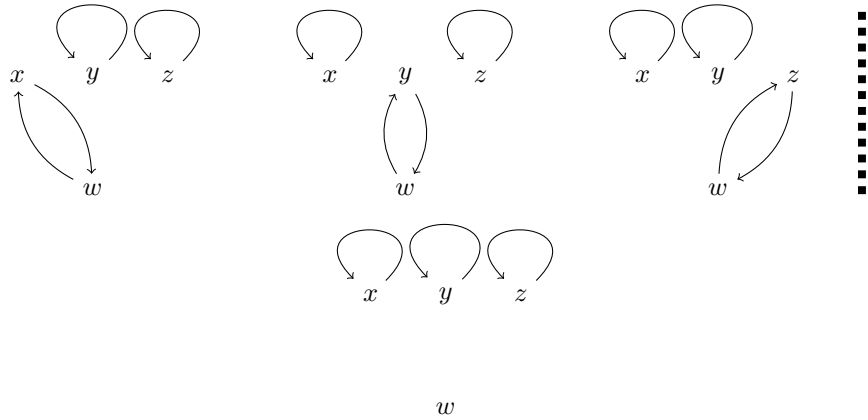
\newcommand{\nodes}{\node (x) at (-1,1.5) {$x$}; \node (y) at (0,1.5) {$y$}; \node (z) at (1,1.5) {$z$}; \node (w) at (0,0) {$w$};}
\newcommand{\selfloop}[1]{\draw [->] (#1) edge[loop] (#1);}
\newcommand{\edgepair}[1]{\draw [->] (#1) edge[bend left] (w); \draw [->] (w) edge[bend left] (#1);}

\begin{center}
\hspace*{-3mm}
\begin{tikzpicture}
\nodes\edgepair{x}\selfloop{y}\selfloop{z}
\end{tikzpicture}
\hspace*{-3mm}
\begin{tikzpicture}
\nodes\edgepair{y}\selfloop{x}\selfloop{z}
\end{tikzpicture}
\hspace*{-3mm}
\begin{tikzpicture}
\nodes\edgepair{z}\selfloop{x}\selfloop{y}
\end{tikzpicture}  
\hspace{4mm}
\begin{tikzpicture}
\draw [dashed] (0,2.5) edge[line width=3pt] (0,0);
\end{tikzpicture}
\hspace*{-1.5mm}
\begin{tikzpicture}
\nodes\selfloop{x}\selfloop{y}\selfloop{z}
\end{tikzpicture}  
\hspace*{-3mm}
\end{center}

\noindent
If we apply the majority rule, then we obtain the graph to the right of the dashed line.
% where the only edges are those connecting the upper three vertices with themselves. 
Thus, the majority rule is not collectively rational with respect to seriality, as each individual graph is serial, but the collective graph is not. Symmetry, on the other hand, is preserved in this example. 
%Finally, this example does not tell us anything about the preservation of reflexivity, as already the individual graphs fail to be reflexive.
\end{example}

\begin{example}[Condorcet paradox]\label{ex:condorcet-paradox}
Now suppose three individuals provide us with the three graphs on the set of vertices $V=\{x,y,z\}$ shown on the lefthand side of the dashed line below:

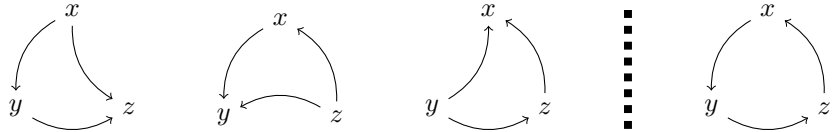
\newcommand{\condorcetnodes}{\node (x) at (0,1.3) {$x$}; \node (y) at (-0.75,0) {$y$}; \node (z) at (0.75,0) {$z$}; }
\newcommand{\condorcetedge}[2]{\draw [->] (#1) edge[bend right] (#2);}
\newcommand{\condorcetedges}[3]{\condorcetedge{#1}{#2}\condorcetedge{#2}{#3}\condorcetedge{#1}{#3}}

\begin{center}
\begin{tikzpicture}
\condorcetnodes
\condorcetedges{x}{y}{z}
\end{tikzpicture}
\qquad
\begin{tikzpicture}
\condorcetnodes
\condorcetedges{z}{x}{y}
\end{tikzpicture}
\qquad
\begin{tikzpicture}
\condorcetnodes
\condorcetedges{y}{z}{x}
\end{tikzpicture}
\qquad
\begin{tikzpicture}
\draw [dashed] (0,1.6) edge[line width=3pt] (0,0);
\end{tikzpicture}
\qquad
\begin{tikzpicture}
\condorcetnodes
\condorcetedge{x}{y}
\condorcetedge{y}{z}
\condorcetedge{z}{x}
\end{tikzpicture}
\end{center}

\noindent
The graph on the righthand side is once again the result of applying the majority rule.
Observe that each of the three input graphs is transitive and complete. So we may interpret these graphs as (strict) preference orders on the candidates $x$, $y$, and $z$. For example, the preferences of the first agent would be $x\succ y\succ z$. The output graph, on the other hand, is not transitive (although it is complete). It does not correspond to a ``rational'' preference, as under that preference we should prefer $x$ to $y$ and $y$ to $z$, but also $z$ to $x$. This is the famous Condorcet paradox described by the Marquis de Condorcet in~1785~\cite{McLeanUrken1995}.
\end{example}

\noindent
So the majority rule is not collectively rational with respect to either seriality or transitivity. On the other hand, we saw that both symmetry and completeness were preserved under the majority rule---at least for the specific examples considered here. In fact, it is not difficult to verify that this was no coincidence, and that the majority rule is collectively rational with respect to a number of properties of interest.

\begin{fact}\label{fact:majority}
The majority rule is collectively rational with respect to 
reflexivity, % proof: follows from general result on unanimity
irreflexivity, % proof: follows from general result on groundedness
symmetry, % proof: see below
and antisymmetry. % proof: if (x,y) has majority support, (y,x) must have minority support
In case $n$, the number of individuals, is odd, the majority rule furthermore is collectively rational with respect to 
completeness 
%strong completeness, 
and connectedness.
% not preserved (checked):
% - transitivity
% - negative transitivity
% - left/right euclidean
% - seriality
\end{fact}

\begin{proof}[Proof sketch]
We give the proofs for symmetry and completeness. The other proofs are very similar.
First, if the input graphs are symmetric, then the set of supporters of edge $(x,y)$ is always identical to the set of supporters of the edge $(y,z)$. Thus, either both or neither have a strict majority.
Second, if the input graphs are complete, then each of them must include at least one of $(x,y)$ and $(y,x)$. Thus, by the pigeon hole principle, when $n$ is odd, at least one of these two edges must have a strict majority.
\end{proof}

\noindent
Rather than establishing further such results for specific aggregation rules, our main interest in this paper is the connection between the axioms satisfied by an aggregation rule and the range of graph properties preserved by the same rule. For some graph properties, collective rationality is easy to achieve, as the following simple \emph{possibility results} demonstrate.

\begin{proposition}\label{prop:unanimity-reflexivity}
Any unanimous aggregation rule is collectively rational with respect to reflexivity.
\end{proposition}

\begin{proof}
If every individual graph includes all edges of the form $(x,x)$, then unanimity ensures the same for the collective graph.
\end{proof}

\begin{proposition}\label{prop:groundedness-irreflexivity}
Any grounded aggregation rule is collectively rational with respect to irreflexivity.
\end{proposition}

\begin{proof}
If no individual graph includes $(x,x)$, then groundedness ensures the same for the collective graph.
\end{proof}

\begin{proposition}\label{prop:neutrality-symmetry}
Any neutral aggregation rule is collectively rational with respect to symmetry.
\end{proposition}

\begin{proof}
If edges $(x,y)$ and $(y,x)$ have the same support, then neutrality ensures that either both or neither will get accepted for the collective graph. 
\end{proof}

\noindent
%Observe that the notion of collective rationality heavily depends on the graph property under consideration: a unanimous aggregator is collectively rational with respect to reflexivity by Proposition~\ref{prop:unanimity-reflexivity}, but there is no guarantee about any other property, e.g., transitivity. To see this it is sufficient to consider the classical example of the Condorcet paradox in preference aggregation \cite{Sen1986}: when the majority rule, which is a unanimous aggregator, is used to aggregate reflexive, transitive and complete graphs, i.e., weak orders, the resulting graph may not be transitive.
Unfortunately, as we will see next, things do not always work out that harmoniously, and certain axiomatic requirements are in conflict with certain collective rationality requirements.

%%%%%%%%%%%%%%%%%%%%%%%%%%%%%%%%%%%%%%%%%%%%%%%%%%%%%%%%%%%%%%%%%%%%%%%%%%%%%%%%
\section{Impossibility Results}\label{sec:impossibility}
%%%%%%%%%%%%%%%%%%%%%%%%%%%%%%%%%%%%%%%%%%%%%%%%%%%%%%%%%%%%%%%%%%%%%%%%%%%%%%%%

\noindent
In social choice theory, an \emph{impossibility theorem} states that it is not possible to devise an aggregation rule that satisfies certain axioms and that is also collectively rational with respect to a certain combination of properties of the structures being aggregated (which in our case are graphs). In this section, we will prove two powerful impossibility theorems for graph aggregation, the \emph{Oligarchy Theorem} and the \emph{Dictatorship Theorem}. The latter identifies a set of requirements that are impossible to satisfy in the sense that the only aggregation rules that meet them are the dictatorships. The former drives on somewhat weaker requirements (specifically, regarding collective rationality) and permits a somewhat larger---but still decidedly unattractive---set of aggregation rules, namely the oligarchies. 

Our results are inspired by---and significantly generalise---the seminal impossibility result for preference aggregation due to Arrow, first published in 1951~\cite{Arrow1963}. We recall Arrow's Theorem in Section~\ref{sec:arrow}. The following subsections are devoted to developing the framework in which to present and then prove our results. Section~\ref{sec:winning-coalitions} introduces \emph{winning coalitions}, i.e., sets of individuals who can force the acceptance or rejection of a given edge, discusses under what circumstances an aggregation rule can be described in terms of a family of winning coalitions, and what structural properties of such a family correspond to either dictatorial or oligarchic aggregation rules. Sections~\ref{sec:contagious} and~\ref{sec:implicative-disjunctive} introduce three so-called \emph{meta-properties} for classifying graph properties and establish fundamental results for these meta-properties. Our impossibility theorems, which are formulated and proved in Section~\ref{sec:general-theorems}, apply to aggregation rules that are collectively rational with respect to graph properties that are covered by some of these meta-properties. Section~\ref{sec:variants}, finally, discusses several variants of our theorems and provides a first illustration of their use.

%%%%%%%%%%%%%%%%%%%%%%%%%%%%%%%%%%%%%%%%%%%%%%%%%%%%%%%%%%%%%%%%%%%%%%%%%%%%%%%%
\subsection{Background: Arrow's Theorem for Preference Aggregation}
\label{sec:arrow}
%%%%%%%%%%%%%%%%%%%%%%%%%%%%%%%%%%%%%%%%%%%%%%%%%%%%%%%%%%%%%%%%%%%%%%%%%%%%%%%%

\noindent
The prime example of an impossibility result is Arrow's Theorem for preference aggregation, with preference relations being modelled as weak orders on some set of alternatives~\cite{Arrow1963}. 
%It states that the only Paretian and independent aggregation rules mapping profiles of weak orders over three or more alternatives to collective weak orders are the dictatorships.
We can reformulate Arrow's Theorem in our framework for graph aggregation as follows:
\begin{quote}
\textit{For $|V|\geq 3$, every unanimous, grounded, and IIE aggregation rule that is collectively rational with respect to reflexivity, transitivity, and completeness must be a dictatorship.}
\end{quote}
Thus, Arrow's Theorem applies to the following scenario. We wish to aggregate the preferences of several agents regarding a set of three or more alternatives. The agents are assumed to express their preferences by ranking the alternatives from best to worst (with indifferences being allowed), i.e., by each providing us with a weak order (a graph that is reflexive, transitive, and complete), and we want our aggregation rule to compute a single such weak order representing a suitable compromise. Furthermore, we want our aggregation rule to respect the basic axioms of unanimity (if all agents agree that $x$ is at least as good as $y$, then the collective preference order should say so), groundedness (if no agent says that $x$ is at least as good as $y$, then the collective preference order should not say so either), and IIE (it should be possible to compute the outcome on an edge-by-edge basis). Arrow's Theorem tells us that this is impossible---unless we are willing to use a dictatorship as our aggregation rule.

This result not only is surprising but also deeply troubling. It therefore is important to understand to what extent similar phenomena arise in other areas of graph aggregation.
%
%Note that we have translated Arrow's (weak) Pareto condition (if every individual ranks $x$ strictly above $y$, then so should the collective) to a combination of unanimity and groundedness. In fact, in the context of the other requirements, Pareto efficiency \emph{implies} both of these properties (so our version is at least as strong as Arrow's Theorem). 
We will revisit Arrow's Theorem in Section~\ref{sec:variants}, where we will also be in a position to explain why the standard formulation of the theorem, given in that section as Theorem~\ref{thm:arrow}, is indeed implied by the variant given here.

In the sequel, we will sometimes refer to aggregation rules that are unanimous, grounded, and IIE as \emph{Arrovian} aggregation rules.

%%%%%%%%%%%%%%%%%%%%%%%%%%%%%%%%%%%%%%%%%%%%%%%%%%%%%%%%%%%%%%%%%%%%%%%%%%%%%%%%
\subsection{Winning Coalitions, Filters, and Ultrafilters}
\label{sec:winning-coalitions}
%%%%%%%%%%%%%%%%%%%%%%%%%%%%%%%%%%%%%%%%%%%%%%%%%%%%%%%%%%%%%%%%%%%%%%%%%%%%%%%%

\noindent
As is well understood in social choice theory, impossibility theorems in preference aggregation heavily feed on independence axioms (in our case IIE). 
Observe that an aggregation rule~$F$ satisfies IIE if and only if for each edge $e\in\Edges$ there exists a set of \emph{winning coalitions} $\mathcal{W}_e\subseteq 2^\N$ such that $e\in F(\prof{E}) \Leftrightarrow N^{\prof{E}}_e\in\mathcal{W}_e$. 
That is, $F$ accepts $e$ if and only if exactly the individuals in one of the winning coalitions for $e$ do. 
Imposing additional axioms on $F$ corresponds to restrictions on the associated family of winning coalitions $\{\mathcal{W}_e\}_{e\in\Edges}$: % Specifically, we get:
\begin{itemize}%[noitemsep]
\item If $F$ is \emph{unanimous}, then $\N\in\mathcal{W}_e$ for any edge~$e$ (i.e., the grand coalition is always a winning coalition). 
\item If $F$ is \emph{grounded}, then $\emptyset\not\in\mathcal{W}_e$ for any edge~$e$ (i.e., the empty set is not a winning coalitions).
\item If $F$ is \emph{monotonic}, then $C_1\in\mathcal{W}_e$ implies $C_2\in\mathcal{W}_e$ for any edge~$e$ and any set $C_2\supset C_1$ (i.e., winning coalitions are closed under supersets).
\item If $F$ is \emph{(NR-)neutral}, then $\mathcal{W}_e=\mathcal{W}_{e'}$ for any two (nonreflexive) edges $e$ and $e'$ (i.e., every edge must have exactly the same set of winning coalitions).
%$\mathcal{W}$ 
\end{itemize}

\noindent
Thus, an aggregation rule that is both IIE and neutral can be fully described in terms of a single set~$\mathcal{W}$ of winning coalitions. Any such $\mathcal{W}$ is a subset of the powerset of $\N$, the set of individuals. The proofs of our impossibility results will exploit the special structure of such subsets of the powerset of $\N$, enforced by both axioms and collective rationality requirements. Specifically, in our proofs we will encounter the concepts of \emph{filters} and \emph{ultrafilters} familiar from model theory~\cite{DaveyPriestley2002}.

%We now prove a first version of our main impossibility result, initially under the additional assumption of neutrality. We do this by proving that the set of winning coalitions corresponding to any aggregator that meets certain conditions is an \emph{ultrafilter}~\cite{DaveyPriestley2002}.

\begin{definition}\label{def:filter}
A \textbf{filter} $\mathcal{W}$ on a set~$\N$ is a collection of subsets of~$\N$ satisfying the following three conditions: 
\begin{enumerate}%[nolistsep]
\item[$(i)$] $\emptyset\not\in\mathcal{W}$;
\item[$(ii)$] $C_1,C_2\in\mathcal{W}$ implies $C_1\cap C_2\in\mathcal{W}$ for any two sets $C_1,C_2\subseteq\N$ 
(closure under intersection); 
\item[$(iii)$] $C_1\in\mathcal{W}$ implies $C_2\in\mathcal{W}$ for any set $C_2\subseteq\N$ with $C_2\supset C_1$ 
(closure under supersets).
\end{enumerate}
\end{definition}

\begin{definition}\label{def:ultrafilter}
An \textbf{ultrafilter} $\mathcal{W}$ on a set~$\N$ is a collection of subsets of~$\N$ satisfying the following three conditions: 
\begin{enumerate}%[nolistsep]
\item[$(i)$] $\emptyset\not\in\mathcal{W}$
\item[$(ii)$] $C_1,C_2\in\mathcal{W}$ implies $C_1\cap C_2\in\mathcal{W}$ for any two sets $C_1,C_2\subseteq\N$ 
(closure under intersection);
\item[$(iii)$] $C$ or $\N\setminus\! C$ is in $\mathcal{W}$ for any set $C\subseteq\N$ (maximality).
\end{enumerate}
\end{definition}

\noindent
Every ultrafilter is a filter; in particular, the ultrafilter conditions imply closure under supersets.
Note that the condition $\emptyset\not\in\mathcal{W}$ directly corresponds to groundedness, while closure under supersets corresponds to monotonicity. 

The use of ultrafilters in social choice theory goes back to the work of \citet{FishburnJET1970} and \citet{KirmanSondermannJET1972}, who employed ultrafilters to prove Arrow's Theorem and its generalisation to an infinite number of individuals. The ultrafilter method also has found applications in judgment aggregation~\cite{HerzbergEckertMSS2012}, and also filters have been used in both preference aggregation~\cite{HanssonPC1976} and judgment aggregation~\cite{GardenforsEP2006}.
% as well as the aggregation of abstract argumentation frameworks~\cite{TohmeEtAlFoIKS2008}. [not really]
The relevance of filters and ultrafilters to aggregation problems is due to the following simple results, which interpret well-known facts from model theory in our specific context.

\begin{lemma}[Filter Lemma]\label{lem:oligarchy-filter}
Let $F$ be an IIE and NR-neutral aggregation rule and let $\mathcal{W}$ be the corresponding set of winning coalitions for nonreflexive edges, i.e., $(x,y)\in F(\prof{E}) \Leftrightarrow N^\prof{E}_{(x,y)}\in\mathcal{W}$ for all $x\not=y\in V$. Then $F$ is NR-oligarchic if and only if $\mathcal{W}$ is a filter.
\end{lemma}

\begin{proof}
$(\Rightarrow)$
Recall that $F$ being NR-oligarchic means that there exists a nonempty coalition $C^\star$ such that a given nonreflexive edge is accepted if and only if all the agents in~$C^\star$ accept it. Thus, the winning coalitions are exactly $C^\star$ and its supersets. This family of sets does not include the empty set and is closed under both intersection and supersets.

$(\Leftarrow)$
Suppose $F$ is determined by the filter~$\mathcal{W}$ as far as nonreflexive edges are concerned. Let $C^\star := \bigcap_{C\in\mathcal{C}} C$, which is well-defined due to $\N$ being finite. Observe that $C^\star$ must be nonempty, due to the first two filter conditions. Now note that $F$ is NR-oligarchic with respect to coalition~$C^\star$.
\end{proof}

\begin{lemma}[Ultrafilter Lemma]\label{lem:dictator-ultrafilter}
Let $F$ be an IIE and NR-neutral aggregation rule and let $\mathcal{W}$ be the corresponding set of winning coalitions for nonreflexive edges, i.e., $(x,y)\in F(\prof{E}) \Leftrightarrow N^\prof{E}_{(x,y)}\in\mathcal{W}$ for all $x\not=y\in V$. Then $F$ is NR-dictatorial if and only if $\mathcal{W}$ is an ultrafilter.
\end{lemma}

\begin{proof}
$(\Rightarrow)$
$F$ being NR-dictatorial means that there exists an  $i^\star\in\N$ such that the winning coalitions for nonreflexive edges are exactly $\{i^\star\}$ and its supersets. This family of sets does not include the empty set, is closed under intersection, and maximal.

$(\Leftarrow)$
Suppose $F$ is determined by the ultrafilter~$\mathcal{W}$ as far as nonreflexive edges are concerned. 
Take an arbitrary $C\in\mathcal{W}$ with $|C|\geq 2$ and consider any nonempty $C'\subsetneq C$. 
By maximality, one of $C'$ and $\N\setminus\! C'$ must be in $\mathcal{W}$.
Thus, by closure under intersection, one of $C\cap C'=C'$ and $C\cap(\N\setminus\! C')=C\setminus\!C'$ must be in $\mathcal{W}$ as well. Observe that both of these sets are nonempty and of lower cardinality than~$C$.
To summarise, we have just shown for any $C\in\mathcal{W}$ with $|C|\geq 2$ at least one nonempty proper subset of $C$ is also in $\mathcal{W}$. 
By maximality, $\mathcal{W}$ is not empty. So take any $C\in\mathcal{W}$.
Due to $\N$ being finite, we can apply our reduction rule a finite number of times to infer that $\mathcal{W}$ must include some singleton $\{i^\star\} \subsetneq \cdots \subsetneq C$. Hence, $F$ is an NR-dictatorship with dictator~$i^\star$.
\end{proof}
%%%%%%%%%%%%%%%%%%%%%%%%%%%%%%%%%%%%%%%%%%%%%%%%%%%%%%%%%%%%%%%%%%%%%%%%%%%%%%%%
\subsection{The Neutrality Axiom and Contagious Graph Properties}\label{sec:contagious}
%%%%%%%%%%%%%%%%%%%%%%%%%%%%%%%%%%%%%%%%%%%%%%%%%%%%%%%%%%%%%%%%%%%%%%%%%%%%%%%%

\noindent
Recall that the neutrality axiom is required to be able to work with a single family of winning coalitions as outlined earlier, yet this axiom does not feature in Arrow's Theorem. As we shall see soon, the reason we do not need to assume neutrality is that, in Arrow's setting, the same restriction on winning coalitions is already enforced by collective rationality with respect to transitivity. This is an interesting link between a specific collective rationality requirement and a specific axiom. In the literature, this fact is often called the \emph{Contagion Lemma}~\cite{Sen1986}, although the connection to neutrality is not usually made explicit.
The same kind of result can also be obtained for other graph properties with a similar structure. 
Let us now develop a definition for a class of graph properties that will allow us to derive neutrality.

Recall that $\CONFIG$ denotes the set of graphs with property~$P$ that include all of the edges in $S^+$ and none of those in $S^-$. We start with a technical definition.

%\begin{definition}\label{def:contagious-prerequisites}
%Let $x,y,x,w\in V$.
%A graph property $P\subseteq 2^{V\times V}\!$ is \textbf{$xy/zw$-contagious} if there exists an edge $e\in V\!\times\! V$ such that 
%$(i)$~for every graph $E\in P$ it is the case that $(x,y),\,e\in E$ implies $(z,w)\in E$ and
%$(ii)$~there exist graphs $E_0,E_1\in P$ such that $E_0\cap\{(x,y),e,(z,w)\}=\{e\}$ and $\{(x,y),e,(z,w)\}\subseteq E_1$.
%\end{definition}

\begin{definition}\label{def:xyzw-contagious}
Let $x,y,z,w\in V$.
A graph property $P\subseteq 2^{\Edges}$ is called \textbf{$\boldsymbol{xy/zw}$-contagious} if there exist two disjoint sets $S^+\!,S^-\subseteq\Edges$ such that the following conditions hold:
\begin{enumerate}
\item[$(i)$] for every graph $E\in\CONFIG$ it is the case that $(x,y)\in E$ implies $(z,w)\in E$; and
\item[$(ii)$] there exist graphs $E_0,E_1\in\CONFIG$ with $(z,w)\not\in E_0$ and $(x,y)\in E_1$.
\end{enumerate}
\end{definition}

\noindent
Part~$(i)$ of Definition~\ref{def:xyzw-contagious} says that, if you accept edge $(x,y)$, then you must also accept edge $(z,w)$---at least if the side condition of you also accepting all the edges in $S^+$ but none of those in $S^-$ is met. 
%That is, an $xy/zw$-contagious property expands the graph from $(x,y)$ to $(z,w)$ under the preconditions expressed by $\CONFIG$. 
That is, the property of $xy/zw$-contagiousness may be paraphrased as the formula $[\bigwedge S^+ \wedge \neg\bigvee S^-] \imp [xEy\imp zEw]$. 
Part~$(ii)$ is a richness condition that says that you have the option of accepting neither or both of $(x,y)$ and $(z,w)$. 
It requires the existence of a graph $E_0$ where neither $(x,y)$ nor $(z,w)$ are accepted, and the existence of a graph $E_1$ where both $(x,y)$ and $(z,w)$ are accepted.

Contagiousness with respect to two given edges will be useful for our purposes if those two edges stand in a specific relationship to each other. The following definition captures the relevant cases.

\begin{definition}\label{def:contagious}
A graph property $P\subseteq 2^{\Edges}$ is called \textbf{contagious} if it satisfies at least one of the three conditions below:  
\begin{enumerate}%[noitemsep]
\item[$(i)$]   $P$ is $xy/yz$-contagious for all triples of distinct vertices $x,y,z\in V$.
\item[$(ii)$]  $P$ is $xy/zx$-contagious for all triples of distinct vertices $x,y,z\in V$.
\item[$(iii)$] $P$ is $xy/xz$-contagious and $xy/zy$-contagious for all triples of distinct vertices $x,y,z\in V$.
%\item[$(i)$]   For some edge $e\in V\!\times\! V$: $(x,y),\,e\in E$ implies $(y,z)\in E$. 
%\item[$(ii)$]  For some edge $e\in V\!\times\! V$: $(x,y),\,e\in E$ implies $(z,x)\in E$. 
%\item[$(iii)$] For some edges $e_1,e_2\in V\!\times\! V$: $(x,y),\,e_1\in E$ implies $(x,z)\in E$, and $(x,y),\,e_2\in E$ implies $(z,y)\in E$. 
\end{enumerate}
\end{definition}

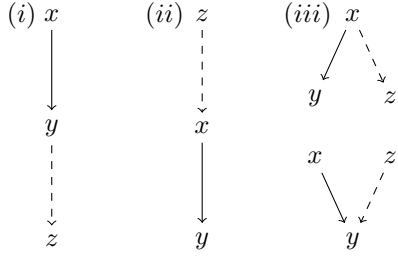
\begin{figure}[t]
\centering
\begin{tikzpicture}[auto]
\node at (-0.4,3) {$(i)$};
\node at (1.5,3) {$(ii)$};
\node at (3.4,3) {$(iii)$};
\node (x1) at (0,3) {$x$};
\node (y1) at (0,1.5) {$y$};
\node (z1) at (0,0) {$z$};
\draw [->] (x1) edge (y1);
\draw [->] (y1) edge[dashed] (z1);
\node (x2) at (2,1.5) {$x$};
\node (y2) at (2,0) {$y$};
\node (z2) at (2,3) {$z$};
\draw [->] (x2) edge (y2);
\draw [->] (z2) edge[dashed] (x2);
\node (x3) at (4,3) {$x$};
\node (y3) at (3.5,1.9) {$y$};
\node (z3) at (4.5,1.9) {$z$};
\draw [->] (x3) edge (y3);
\draw [->] (x3) edge[dashed] (z3);
\node (x4) at (3.5,1.1) {$x$};
\node (y4) at (4,0) {$y$};
\node (z4) at (4.5,1.1) {$z$};
\draw [->] (x4) edge (y4);
\draw [->] (z4) edge[dashed] (y4);
\end{tikzpicture}
\caption{Illustration of Definition~\ref{def:contagious}, indicating given edges (solid) and implied edges (dashed).\label{fig:contagious}}
\end{figure}

\noindent
That is, Definition~\ref{def:contagious} covers pairs of edges where $(i)$~the second edge is a successor of the first edge, where $(ii)$~the second edge is a predecessor of the first edge, and where $(iii)$~the two edges share either a starting point or an end point.
This covers all cases of two edges meeting in one point.
The three cases are illustrated in Figure~\ref{fig:contagious}.
As will become clear in the proof of Lemma~\ref{lem:neutrality}, case~$(iii)$ differs from the other two, as only one of these two types of connections would not be sufficient to ``traverse'' the full graph.

%The property of being right Euclidean is an example for a contagious graph property, as it satisfies condition~$(i)$. To see this, let the edge $e$ postulated by Definition~\ref{def:contagious-prerequisites} be equal to $(x,z)$. The richness condition of Definition~\ref{def:contagious-prerequisites} is also satisfied, as the Euclidean propert does not force us to reject either $(x,y)$ or $e=(x,z)$.
%Being left Euclidean is contagious by condition~$(ii)$, with $e=(z,y)$.
%Transitivity, finally, satisfies condition~$(iii)$: let $e=(y,z)$  for $xy/zx$-contagiousness and $e=(z,x)$ for $xy/zy$-contagiousness.  

\begin{fact}\label{fact:contagiousness}
For $|V|\geq 3$, the two Euclidean properties, transitivity, negative transitivity, and connectedness are all contagious graph properties.
\end{fact}

\begin{proof}
Let us first consider the property of being a right-Euclidean graph. It satisfies condition~$(i)$ of Definition~\ref{def:contagious}. To prove this, we will show that the right-Euclidean property is $xy/yz$-contagious for all triples $x,y,z\in V$. Let $S^+=\{(x,z)\}$ and $S^-=\emptyset$, i.e., $\CONFIG$ is the set of all right-Euclidean graphs containing $(x,z)$. 
Condition~$(i)$ of Definition~\ref{def:xyzw-contagious} is met: any graph in $\CONFIG$ contains $(x,z)$; therefore, by the right-Euclidean property, $(y,z)$ needs to be accepted whenever $(x,y)$ is. 
Condition~$(ii)$ is also satisfied. Let $E_0$ be the graph only containing the single edge $(x,z)$, and let $E_1$ be the graph containing exactly the three edges $(x,y)$, $(y,z)$, and $(x,z)$. Both graphs are right-Euclidean and, since they include $(x,z)$, they also belong to $\CONFIG$.

An alternative way of seeing that the right-Euclidean property is contagious is to observe that it is equivalent to the formula $[xEz] \imp [xEy\imp yEz]$, with all variables universally quantified.
Similarly, the left-Euclidean property, which can be rewritten as $[zEy] \imp [xEy\imp zEx]$, is contagious by condition~$(ii)$.
Transitivity satisfies condition~$(iii)$, as we can rewrite it either as $[yEz]\imp [xEy\imp xEz]$ or as $[zEx]\imp [xEy\imp zEy]$.
Negative transitivity can be rewritten either as $[\neg(zEy)] \imp [xEy \imp xEz]$ or as $[\neg(xEz)] \imp [xEy \imp zEy]$, and this property thus also satisfies condition~$(iii)$. 
Connectedness, finally, can be rewritten as $[xEz \wedge \neg zEy] \imp [xEy\imp yEz]$ and thus satisfies condition~$(i)$.
For all these properties, the richness conditions are easily verified to hold as well.
\end{proof}

\noindent
We are now ready to prove a powerful lemma, the \emph{Neutrality Lemma}, showing that any Arrovian aggregation rule that is collectively rational with respect to a contagious graph property must be neutral (at least as far as nonreflexive edges are concerned). This generalises a result often referred to as the \emph{Contagion Lemma} in the literature~\cite{Sen1986}.

\begin{lemma}[Neutrality Lemma]\label{lem:neutrality}
For $|V|\geq 3$, any unanimous, grounded, and IIE aggregation rule that is collectively rational with respect to a contagious graph property must be NR-neutral.
\end{lemma}

\begin{proof}
We will first establish a generic result for collective rationality with respect to $xy/zw$-contagiousness. Let $x,y,z,w\in V$. Take any graph property $P$ that is $xy/zw$-contagious and take any aggregation rule~$F$ that is unanimous, grounded, IIE, and collectively rational with respect to~$P$. Let $\{\mathcal{W}_e\}_{e\in\Edges}$ be the family of winning coalitions associated with $F$. We want to show that $\mathcal{W}_{(x,y)} \subseteq \mathcal{W}_{(z,w)}$. 
So let $C$ be a coalition in $\mathcal{W}_{(x,y)}$. Let $S^+\!,S^-\subseteq\Edges$ and $E_0,E_1\in\CONFIG$ be defined as in Definition~\ref{def:xyzw-contagious}. 
Consider a profile $\prof{E}$ in which the individuals in $C$ propose graph $E_1$ and all others propose $E_0$. That is, all individuals accept the edges in $S^+\!$, none accept any of those in $S^-\!$, exactly the individuals in $C$ accept edge $(x,y)$, and exactly those in $C$ also accept $(z,w)$. Now consider the collective graph $F(\prof{E})$. By unanimity $S^+\subseteq F(\prof{E})$, by groundedness $S^-\cap F(\prof{E})=\emptyset$, and finally $(x,y)\in F(\prof{E})$ due to $C$ being a winning coalition for $(x,y)$. By collective rationality, $F(\prof{E})\in P$ and thus also $F(\prof{E})\in\CONFIG$. But then, due to $xy/zw$-contagiousness of $F(\prof{E})$, we get $(z,w)\in F(\prof{E})$. As it was exactly the individuals in $C$ who accepted $(z,w)$, coalition $C$ must be winning for $(z,w)$, i.e., $C\in\mathcal{W}_{(z,w)}$, and we are done.

We are now ready to prove the lemma.
Take any graph property $P$ that is contagious and take any aggregation rule~$F$ that is unanimous, grounded, IIE, and collectively rational with respect to~$P$. Let $\{\mathcal{W}_e\}_{e\in\Edges}$ be the family of winning coalitions associated with $F$. We need to show that there exists a unique $\mathcal{W}\subseteq2^\N$ such that $\mathcal{W}=\mathcal{W}_e$ for every nonreflexive edge~$e$. 
By unanimity, the sets $\mathcal{W}_e$ are not empty (because at least $\N\in\mathcal{W}_e$). Consider any three vertices $x,y,z\in V$ and any coalition $C\in\mathcal{W}_{(x,y)}$. We will show 
%$\mathcal{W}_{(x,y)} \subseteq \mathcal{W}_{(y,z)}$ and $\mathcal{W}_{(x,y)} \subseteq \mathcal{W}_{(y,x)}$, i.e., if $C$ is winning for $(x,y)$ then it 
that $C$ is also winning for both $(y,z)$ and $(y,x)$. 
%We will employ collective rationality to show that then also $C\in\mathcal{W}_{(y,z)}$. That is, we will show that, if $C$ is winning for $(x,y)$, then $C$ is also winning for its `successor edge' $(y,z)$. 
If we can show this for any $x, y, z$, then we are done, as we can then repeat the same method several times until all nonreflexive edges are covered.

For each of the three possible ways in which $P$ can be contagious (see Definition~\ref{def:contagious}), we will use different instances of our generic result for $xy/zw$-contagiousness above:
\begin{itemize}
\item First, if $P$ is contagious by virtue of condition~$(i)$, then we can use $xy/yz$-contagiousness to get $C\in\mathcal{W}_{(y,z)}$ and its instance $xy/yx$-contagiousness (with $z:=x$) to obtain also $C\in\mathcal{W}_{(y,x)}$.
\item Second, if $P$ is contagious due to condition~$(ii)$, we use $xy/yx$-contagiousness to get $C\in\mathcal{W}_{(y,x)}$, and then $yx/zy$-contagiousness to get $C\in\mathcal{W}_{(z,y)}$ and $zy/yz$-contagiousness to get $C\in\mathcal{W}_{(y,z)}$. 
\item Third, suppose $P$ is contagious by virtue of condition~$(iii)$. We first use $xy/zy$-contagiousness to obtain $C\in\mathcal{W}_{(z,y)}$ and then $zy/zx$-contagiousness to get $C\in\mathcal{W}_{(z,x)}$. From the latter, via $zx/yx$-contagiousness we get $C\in\mathcal{W}_{(y,x)}$. Finally, $yx/yz$-contagiousness then entails $C\in\mathcal{W}_{(y,z)}$.
\end{itemize}
Hence, we obtain the required transfer from one edge $(x,y)$ to both its successor $(y,z)$ and its inverse $(y,x)$ in all three cases, and our proof is complete.
\end{proof}

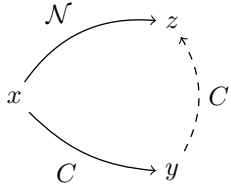
\begin{figure}[t]
\centering
\begin{tikzpicture}[auto]
  \node          (a)          {$x$};
  \node          (b) at (2.1,1) {$z$};
  \node          (c) at (2.1,-1) {$y$};

  \path[->] (a) edge [->, line width=0.5pt,bend left=30]        node       {$\N$} (b)
            (b) edge [<-, line width=0.5pt,bend left=30, dashed]        node        {$C$} (c) 
            (c) edge [<-,bend left=20, line width=0.5pt] node {$C$} (a);
\end{tikzpicture}
\caption{Collective rationality with respect to the right-Euclidean property implies neutrality.\label{fig:neutrality}}
\end{figure}

\noindent
Figure~\ref{fig:neutrality} provides an illustration of a specific instance of the main argument in the proof of Lemma~\ref{lem:neutrality} when the right-Euclidean property is considered, which is $xy/yz$-contagious by Fact~\ref{fact:contagiousness}.
We have $S^+=\{(x,z)\}$ and $S^-=\emptyset$. $E_1$ is the graph that accepts all three edges $(x,y)$, $(y,z)$ and $(x,z)$, and $E_0$ accepts only edge $(x,z)$. 
Consider profile $\prof{E}$, in which the individuals in $C$ choose $E_1$ and all others choose $E_0$.
That is, the individuals in $C$ accept $(x,y)$ and $(y,z)$, while $(x,z)$ is accepted by all individuals in $\N$. 
By unanimity, $(x,z)$ must be accepted, and due to $C\in \mathcal{W}_{(x,y)}$ also $(x,y)$ should be accepted. We can now conclude, since $F$ is collectively rational with respect to the right-Euclidean property, that $(y,z)$ should also be accepted, and hence that $C\in \mathcal{W}_{(y,z)}$.
It is then sufficient to consider all triples to obtain neutrality over all (nonreflexive) edges.

%%%%%%%%%%%%%%%%%%%%%%%%%%%%%%%%%%%%%%%%%%%%%%%%%%%%%%%%%%%%%%%%%%%%%%%%%%%%%%%%
\subsection{Implicative and Disjunctive Graph Properties}
\label{sec:implicative-disjunctive}
%%%%%%%%%%%%%%%%%%%%%%%%%%%%%%%%%%%%%%%%%%%%%%%%%%%%%%%%%%%%%%%%%%%%%%%%%%%%%%%%

\noindent
Let us briefly recapitulate where we are at this point. We now know that any Arrovian aggregation rule~$F$ that is collectively rational with respect to some contagious graph property~$P$ can be fully described in terms of a single family~$\mathcal{W}$ of winning coalitions, at least as far as $F$'s behaviour on nonreflexive edges is concerned. To prove our impossibility results, we will need to derive structural properties of $\mathcal{W}$ that allow us to infer that $\mathcal{W}$ is either a filter or an ultrafilter (so we can use Lemma~\ref{lem:oligarchy-filter} or~\ref{lem:dictator-ultrafilter}, respectively). 
%Some of these structural properties will come from additional axioms we will assume (namely, as we have already seen in Section~\ref{sec:winning-coalitions}, groundedness of $F$ implies $\emptyset\not\in\mathcal{W}$). But most importantly, most of 
These structural properties will be shown to follow from collective rationality requirements with respect to graph properties belonging to a certain class of such properties. 

We will now introduce two such classes of graph properties, or ``meta-properties'' as we shall also call them. Recall that we have already seen one meta-property, namely contagiousness (which, however, is much more complex than the following meta-properties). First, a graph property is \emph{implicative} if the inclusion of some edges can force the inclusion of a further edge, as is the case, for instance, for transitivity. The following definition makes this precise.

\begin{definition}\label{def:implicative}
A graph property $P\subseteq 2^{\Edges}$ is called \textbf{implicative} if there exist two disjoint sets $S^+\!,S^-\subseteq\Edges$ and three distinct edges $e_1,e_2,e_3\in\Edges\setminus(S^+\cup S^-)$ such that the following conditions hold:
\begin{enumerate}
\item[$(i)$] for every graph $E\in\CONFIG$ it is the case that $e_1,e_2\in E$ implies $e_3\in E$; and
\item[$(ii)$]there exist graphs $E_0, E_1, E_2, E_{13}, E_{123}\in \CONFIG$ with 
$E_0\cap\{e_1,e_2,e_3\}=\emptyset$, 
$E_1\cap\{e_1,e_2,e_3\}=\{e_1\}$, 
$E_2\cap\{e_1,e_2,e_3\}=\{e_2\}$, 
$E_{13}\cap\{e_1,e_2,e_3\}=\{e_1,e_3\}$, and
$\{e_1,e_2,e_3\}\subseteq E_{123}$. 
\end{enumerate}
\end{definition}

\noindent
Part~$(i)$ expresses that all graphs with property $P$ (that also include all edges in $S^+$ and none from $S^-$) must satisfy the formula $e_1\wedge e_2\rightarrow e_3$.
Part~$(ii)$ is a richness condition saying that accepting/rejecting any combination of $e_1$ and $e_2$ is possible, that $e_3$ need not be accepted unless both $e_1$ and $e_2$ are, and that $e_3$ can be accepted even if only the first antecedent~$e_1$ is.\footnote{In our earlier work, we did not require the existence of $E_{13}$~\citep{EndrissGrandiECAI2014}. The slightly stronger formulation used here is necessary to prove one of our general impossibility theorems (Theorem~\ref{thm:oligarchy}), but not the other (Theorem~\ref{thm:dictator}).}
Observe that Definition~\ref{def:implicative} has an existential form, i.e., we simply need to find two subsets $S^+$ and $S^-$ for the precondition, and three edges $e_1$, $e_2$ and $e_3$ that satisfy the two requirements $(i)$ and~$(ii)$. In this sense, implicativeness is much less demanding than contagiousness, which imposes conditions across the entire graph. 
%The third  meta-property to be introduced next also has this much simpler existential form.
%This meta-property for graph properties
Implicativeness may be paraphrased as the formula $[\bigwedge S^+ \wedge \neg\bigvee S^-] \imp [e_1 \wedge e_2 \imp e_3]$.

\begin{fact}\label{fact:implicativeness}
For $|V|\geq 3$, the two Euclidean properties, transitivity, and connectedness are all implicative graph properties.
\end{fact}

\begin{proof}[Proof (sketch)]
Let $V=\{v_1,v_2,v_3,\ldots\}$.
To see that transitivity satisfies Definition~\ref{def:implicative}, choose $S^+=S^-=\emptyset$, $e_1=(v_1,v_2)$, $e_2=(v_2,v_3)$, and $e_3=(v_1,v_3)$. 
Transitivity implies that, if both $e_1$ and $e_2$ are accepted, then also $e_3$ should be accepted.
All remaining acceptance/rejection patterns of $e_1$, $e_2$, and $e_3$ are possible, in accordance with condition~$(ii$).
The proofs for the Euclidean properties are similar.
Rewriting connectedness as $[\neg yEz] \imp  [(xEy\wedge xEz) \imp zEy]$ shows that it is implicative as well.
\end{proof}

\noindent
Note that implicativeness is a very weak requirement: even transitivity restricted to a single triple of edges is sufficient to satisfy it. Next, we define \emph{disjunctive} graph properties as properties that force us to include at least one of two given edges, as is the case, for instance, for completeness.

%\begin{definition}
%A graph property $P\subseteq 2^{V\times V}\!$ is \textbf{disjunctive} if there exists a set $S\subseteq V\times V$ with $|S|\geq 2$ such that 
%$(i)$~$S\cap E\not=\emptyset$ for all graphs $E\in P$ and
%$(ii)$~there exist two graphs $E_1,E_2\in P$ with 
%$|E_1\cap S| = 1$,
%$|E_2\cap S| = 1$, and
%$E_1\cap S \not= E_2\cap S$.
%\end{definition}
%\noindent
%Intuitively speaking, for $S=\{e_1,\ldots,e_k\}$ part~$(i)$ ensures that all graphs with property $P$ `satisfy' the formula $e_1\vee\cdots\vee e_k$.  
%Part~$(ii)$ is a richness condition ensuring that there are at least two graphs that each include only a single edge from $S$ and that these two edges are not the same.
%Examples for disjunctive properties are completeness (with $S=\{(x,y),(y,x)\}$ for arbitrary $x,y\in V$) and seriality (with $S=\{(x,y)\mid y\in V\}$ for arbitrary $x\in V$).
%Like implicativeness, disjunctiveness is a very weak (and `local') requirement, as it only postulates the existence of one suitable set~$S$. 

\begin{definition}\label{def:disjunctive}
A graph property $P\subseteq 2^{\Edges}$ is called \textbf{disjunctive} if there exist two disjoint sets $S^+\!,S^-\subseteq\Edges$ and two distinct edges $e_1, e_2\in\Edges \setminus (S^+\cup S^-)$ such that the following conditions hold:
\begin{enumerate}
\item[$(i)$] for every graph $E\in\CONFIG$ we have $e_1\in E$ or $e_2\in E$; and
\item[$(ii)$] there exist two graphs $E_1,E_2\in\CONFIG$ with $E_1\cap\{e_1,e_2\}=\{e_1\}$ and $E_2\cap\{e_1,e_2\}=\{e_2\}$.
\end{enumerate}
\end{definition}

\noindent
Part~$(i)$ ensures that all graphs with property $P$ (that meet the precondition of including all edges in $S^+$ and none from $S^-$) satisfy the formula $e_1\vee e_2$.
Part~$(ii)$ is a richness condition ensuring that there are at least two graphs that each include only one of $e_1$ and $e_2$.
Definition~\ref{def:disjunctive} also has an existential form, and it may be paraphrased as the formula $[\bigwedge S^+ \wedge \neg\bigvee S^-] \imp [e_1 \vee e_2]$.

\begin{fact}\label{fact:disjunctiveness}
For $|V|\geq 3$, negative transitivity, connectedness, completeness, %strong completeness, 
nontriviality, and seriality are all disjunctive graph properties.
\end{fact}

\begin{proof}
Let $V=\{v_1,\ldots,v_m\}$. 
For negative transitivity, choose $S^+=\{v_1,v_2\}$, $S^-=\emptyset$, $e_1=(v_1,v_3)$, and $e_2=(v_3,v_2)$
to see that the conditions are satisfied. 
For connectedness, choose $S^+=\{(v_1,v_2),(v_1,v_3)\}$, $S^-=\emptyset$, $e_1=(v_2,v_3)$, and $e_2=(v_3,v_2)$.
For completeness, %and strong completeness, 
choose $S^+=S^-=\emptyset$, $e_1=(v_1,v_2)$, and $e_2=(v_2,v_1)$. 
For nontriviality, choose $S^+=\emptyset$, $S^-=\{(v_i,v_j) : \{i,j\}\not=\{1,2\}\}$, $e_1=(v_1,v_2)$, and $e_2=(v_2,v_1)$.
Finally, for seriality, choose $S^+=\emptyset$, $S^-=\{(v_1,v_1),(v_1,v_2),\ldots,(v_1,v_{m-2})\}$, $e_1=(v_1,v_{m-1})$, and $e_2=(v_1,v_m)$.
\end{proof}

\noindent
Note that some of these results could be strengthened to the case of $|V|=2$, but doing so would not be useful for our purposes here.

%Finally, an \emph{exclusive} graph property is a property that forces as to exclude at least one of two given edges. An example is acyclicity.
%
%\begin{definition}\label{def:exclusive}
%A graph property $P\subseteq 2^{\Edges}$ is called \textbf{exclusive} if there exist two disjoint sets $S^+,S^-\subseteq\Edges$ and two distinct edges $e_1, e_2\in\Edges \setminus (S^+\cup S^-)$ such that the following conditions hold:
%\begin{enumerate}
%\item[$(i)$] for every graph $E\in\CONFIG$ we have $e_1\not\in E$ or $e_2\not\in E$; and
%\item[$(ii)$] there exist two graphs $E_1,E_2\in\CONFIG$ with $E_1\cap\{e_1,e_2\}=\{e_1\}$ and $E_2\cap\{e_1,e_2\}=\{e_2\}$.
%\end{enumerate}
%\end{definition}
%
%\noindent
%Thus, Definition~\ref{def:exclusive} corresponds to the formula $[\bigwedge S^+ \wedge \neg\bigvee S^-] \imp [\neg e_1 \vee \neg e_2]$. %The richness condituion~$(ii)$ requires the existence of two graphs that each include one of the two special edges.
%
%\begin{fact}
%For $|V|\geq 2$, acyclicity is an exclusive graph property.
%\end{fact}
%
%\begin{proof}
%Let $V=\{v_1,\ldots,v_m\}$ and choose $S^+=\emptyset$, $S^-=\emptyset$, $e_1=(v_1,v_2)$, and $e_2=(v_2,v_1)$.
%\end{proof}

%%%%%%%%%%%%%%%%%%%%%%%%%%%%%%%%%%%%%%%%%%%%%%%%%%%%%%%%%%%%%%%%%%%%%%%%%%%%%%%%
\subsection{Two General Impossibility Theorems for Graph Aggregation}
\label{sec:general-theorems}
%%%%%%%%%%%%%%%%%%%%%%%%%%%%%%%%%%%%%%%%%%%%%%%%%%%%%%%%%%%%%%%%%%%%%%%%%%%%%%%%

\noindent
We are now ready to present our impossibility results. We will prove two main theorems. What they have in common is that they talk about Arrovian aggregation rules~$F$ that are collectively rational with respect to a graph property~$P$ that is contagious and implicative. % (and thus neutral, by Lemma~\ref{lem:neutrality}). 
%They differ in the additional assumptions made. 
For the first theorem, %we also assume that $F$ is monotonic, and 
we will show that under these assumptions $F$ must be oligarchic (at least as far as nonreflexive edges are concerned). For the second theorem, we also assume that $P$ is disjunctive, and show that then $F$ must be dictatorial (at least on nonreflexive edges). 

%Arrovian + CR for exclusive:
%intersection of all winning coalitions over all edges (need not be same) is set of agnets who can veto every edge
%can use CR for exclusiveness to show that this intersection is nonempty: how? ... not really, only between certain edges
%do not need monotonicity

\begin{theorem}[Oligarchy Theorem]\label{thm:oligarchy}
For $|V|\geq 3$, any unanimous, grounded, %monotonic, 
and IIE aggregation rule~$F$ that is collectively rational with respect to a graph property~$P$ that is contagious and implicative must be oligarchic on nonreflexive edges.
\end{theorem}

\begin{proof}
Take any graph property~$P$ that is contagious and implicative, 
and any aggregation rule~$F$ that is unanimous, grounded, %monotonic, 
IIE, and collectively rational with respect to~$P$.
By Lemma~\ref{lem:neutrality}, $F$ must be NR-neutral.
Hence, there exists a set of winning coalitions $\mathcal{W}\subseteq 2^\N$ determining $F$ in the sense that $e\in F(\prof{E}) \Leftrightarrow N^{\prof{E}}_e\in\mathcal{W}$ for any nonreflexive edge~$e$. 

We shall prove that $\mathcal{W}$ is a filter (see Definition~\ref{def:filter}), from which the theorem then follows by Lemma~\ref{lem:oligarchy-filter}.
Condition~$(i)$ holds, as $F$ is grounded. %; and condition~$(iii)$ holds, as $F$ is monotonic.
So we still need to show that $\mathcal{W}$ satisfies condition~$(ii)$, i.e., that it is closed under intersection, and condition~$(iii)$, i.e., that it is closed under supersets.
To do so, we will make use of the assumption that $P$ is implicative.
Let $S^+\!,S^-\subseteq\Edges$ and $e_1,e_2,e_3\in \Edges$; and let $E_0,E_1,E_2,E_{13},E_{123}\in\CONFIG$ be defined as in Definition~\ref{def:implicative}. 

First, take any two winning coalitions $C_1,C_2\in\mathcal{W}$.
Consider a profile of graphs $\prof{E}$ satisfying $P$ in which exactly the individuals in $C_1\cap C_2$ propose $E_{123}$, those in $C_1\setminus\! C_2$ propose $E_1$, those in $C_2\setminus\! C_1$ propose $E_2$, and all others propose $E_0$. Thus, exactly the individuals in $C_1$ accept $e_1$, exactly those in $C_2$ accept $e_2$, and exactly those in $C_1\cap C_2$ accept $e_3$. Furthermore, all individuals accept $S^+$ and all of them reject $S^-\!$.
Hence, due to unanimity, all edges in $S^+$ must be part of the collective graph $F(\prof{E})$, while due to groundedness, none of the edges in $S^-$ can be part of $F(\prof{E})$.
As $F$ is collectively rational with respect to $P$, we get $F(\prof{E})\in\CONFIG$. 
Now, since $C_1$ and $C_2$ are winning coalitions, $e_1$ and $e_2$ must be part of $F(\prof{E})$.
As $P$ is implicative, this means that $e_3\in F(\prof{E})$. 
Hence, we must have $C_1\cap C_2\in\mathcal{W}$, i.e., $\mathcal{W}$ is closed under intersection.

Now, take any winning coalition $C_1\in\mathcal{W}$ and any other coalition $C_2$ with $C_1\subseteq C_2$.
Consider a profile of graphs $\prof{E}$ satisfying $P$ in which the individuals in $C_1$ propose $E_{123}$, those in $C_2\setminus\! C_1$ propose $E_{13}$, and those in $\N\setminus\! C_2$ propose $E_1$. In other words, the coalition of supporters of $e_1$ is $\N$, the coalition of supporters of $e_2$ is $C_1$, the coalition of supporters of $e_3$ is $C_2$, all individuals accept $S^+$, and all of them also reject $S^-$.
Due to unanimity and as $C_1\in\mathcal{W}$, $e_1$ and $e_2$ will be part of the collective graph~$F(\prof{E})$. As $F$ is collectively rational with respect to $P$, we thus also get $e_3\in F(\prof{E})$. 
Hence, as $e_3$ was supported by $C_2$, it must be the case that $C_2\in\mathcal{W}$, i.e., $\mathcal{W}$ is closed under supersets.
\end{proof}

\noindent
If the graph property to be preserved under aggregation also is required to be disjunctive, we can further tighten this impossibility result and obtain a dictatorship. 
%Monotonicity turns out to become redundant in this case, so is not listed amongst our assumptions in the next theorem. 
The proof is very similar to that of Theorem~\ref{thm:oligarchy}, the only added difficulty being that of proving maximality of the filter from collective rationality with respect to a disjunctive graph property.

\begin{theorem}[Dictatorship Theorem]\label{thm:dictator}
For $|V|\geq 3$, any unanimous, grounded, and IIE aggregation rule~$F$ that is collectively rational with respect to a graph property~$P$ that is contagious, implicative, and disjunctive must be dictatorial on nonreflexive edges.
\end{theorem}

\begin{proof}
Take any graph property~$P$ that is contagious, implicative, and disjunctive,
and any aggregation rule~$F$ that is unanimous, grounded, and IIE, and collectively rational with respect to~$P$.
By Lemma~\ref{lem:neutrality}, $F$ must be NR-neutral, i.e., on nonreflexive edges, $F$ must be determined by a single family~$\mathcal{W}$ of winning coalitions. 
We shall prove that the $\mathcal{W}$ is an ultrafilter (see Definition~\ref{def:ultrafilter}), from which the theorem then follows by Lemma~\ref{lem:dictator-ultrafilter}.
Condition~$(i)$ holds, as $F$ is grounded. Condition~$(ii)$ follows from $P$ being implicative and can be proved exactly as for Theorem~\ref{thm:oligarchy}.

To derive condition~$(iii)$, we will make use of the assumption that $P$ is disjunctive. 
Let $S^+\!,S^-\subseteq\Edges$ and $e_1,e_2\in\Edges$; and let $E_1,E_2\in\CONFIG$ be defined as in Definition~\ref{def:disjunctive}.
Now take any winning coalition $C\in\mathcal{W}$.
Consider a profile $\prof{E}$ satisfying $P$ in which exactly the individuals in $C$ propose $E_1$ and exactly those in $\N\setminus\! C$ propose $E_2$.
Recall that $S^+\subseteq E_1$ and $S^+\subseteq E_2$, i.e., all individuals accept $S^+\!$. Thus, due to unanimity, all of the edges in $S^+$ must be part of the collective graph $F(\prof{E})$.
Analogously, due to groundedness, none of the edges in $S^-$ can be part of $F(\prof{E})$. 
Thus, as $F$ is collectively rational with respect to $P$, we get $F(\prof{E}) \in \CONFIG$. As $P$ is disjunctive, this means that one of $e_1$ and $e_2$ has to be part of $F(\prof{E})$.
Hence, $C\in\mathcal{W}$ or $(\N\setminus\! C)\in\mathcal{W}$.
\end{proof}

%The properties of contagiousness and implicativeness have a similar structure. They both speak about acceptance of two particular edges implying the acceptance of a third. Contagiousness is more demanding in that it applies to \emph{all} triples of edges meeting certain conditions, while implicativeness only postulates the existence of one such triple. On the other hand, implicativeness comes with a slightly more demanding richness condition. In practice, most natural graph properties will either be both contagious and implicative, or they will be neither. This certainly is the case for all of the standard properties listed in Table~\ref{tab:graph-properties}. To account for this empirical fact, we introduce some further terminology and define a graph property to be \emph{strongly implicative} if it is both contagious and implicative.

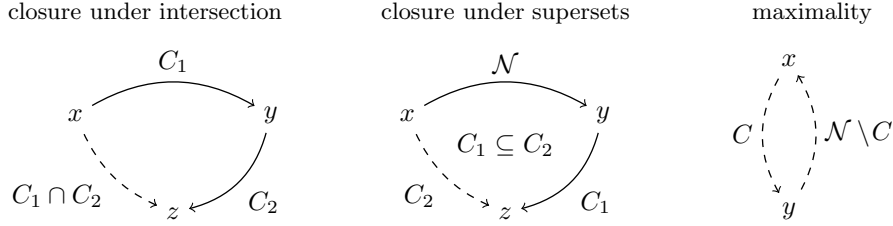
\begin{figure}[t]
\centering
\begin{tabular}{ccccc}
{\small closure under intersection} && {\small closure under supersets} && {\small maximality} \\[5pt]
\begin{tikzpicture}[auto]
  \node          (a)          {$x$};
  \node          (b) at (2.6,0) {$y$};
  \node          (c) at (1.3,-1.3) {$z$};
  \path[->] (a) edge [->, line width=0.5pt,bend left=30]        node       {$C_1$} (b)
            (b) edge [->, line width=0.5pt,bend left=30]        node        {$C_2$} (c) 
            (c) edge [<-,dashed,bend left=20, line width=0.5pt, dashed] node {$C_1\cap C_2$} (a);
\end{tikzpicture}
& \hspace{0.4cm} &
\begin{tikzpicture}[auto]
  \node          (middle) at (1.3,-0.4) {$C_1\subseteq C_2$};
  \node          (a)          {$x$};
  \node          (b) at (2.6,0) {$y$};
  \node          (c) at (1.3,-1.3) {$z$};
  \path[->] (a) edge [->, line width=0.5pt,bend left=30]        node       {$\N$} (b)
            (b) edge [->, line width=0.5pt,bend left=30]        node        {$C_1$} (c) 
            (c) edge [<-,dashed,bend left=20, line width=0.5pt, dashed] node {$C_2$} (a);
\end{tikzpicture}
& \hspace{0.4cm} &
\begin{tikzpicture}[auto]
  \node          (a)        {$x$};
  \node          (b) at (0,-2) {$y$};  
  \path[->] (b) edge [<-, dashed, line width=0.5pt,bend left=30]  node  {$C$} (a)
            (a) edge [<-, dashed, line width=0.5pt,bend left=30]  node  {$\N\setminus\! C$} (b);          
\end{tikzpicture}
\end{tabular}
\caption{Using collective rationality with respect to transitvity and completeness.\label{fig:ultrafilter}}
\end{figure}

\noindent
It may be helpful to illustrate the main arguments in the proofs of Theorem~\ref{thm:oligarchy} and~\ref{thm:dictator} by instantiating them for specific graph properties rather than generic meta-properties. 
For instance, we can derive closure of intersection of $\mathcal{W}$ by using collective rationality with respect to transitivity, which by Fact~\ref{fact:implicativeness} is an implicative property.
Consider the profile depicted on the left in Figure~\ref{fig:ultrafilter}, in which exactly the individuals in $C_1$ accept edge $e_1=(x,y)$, exactly those in $C_2$ accept $e_2=(y,z)$, and exactly those in $C_1\cap C_2$ accept $e_3=(x,z)$. As both $C_1$ and $C_2$ are winning coalitions, we obtain that both $(x,y)$ and $(y,z)$ need to be collectively accepted. We can now conclude, since $F$ is collectively rational with respect to transitivity, that the edge $(x,z)$ should also be accepted. Hence, the coalition accepting $(x,z)$, which is $C_1\cap C_2$, must be a winning coalition as well.
Similarly, we can obtain closure under supersets from collective rationality with respect to transitivity using the profile shown in the middle of Figure~\ref{fig:ultrafilter}. Here, all individuals accept $e_1=(x,y)$, those in $C_1$ accept $e_2=(y,z)$, and those in $C_2$, which is a superset of $C_1$, accept $e_3=(x,z)$. As both $\N$ and $C_1$ are winning coalitions, both $(x,y)$ and $(y,z)$ get accepted. Thus, as $F$ is collectively rational with respect to transitivity, so does $(x,z)$. Hence, $C_2$, the coalition of supportes of $(x,z)$, must also be winning. 
Finally, we can prove maximality of $\mathcal{W}$ by using collective rationality with respect to, say, completeness, which by Fact~\ref{fact:disjunctiveness} is a disjunctive property. 
Consider the profile on the righthand side of Figure~\ref{fig:ultrafilter}, in which exactly the individuals in $C$ accept $e_1=(x,y)$ and exactly those in $\N\setminus\! C$ accept $e_2=(y,x)$. As $F$ is collectively rational with respect to completeness, one of the two edges has to get accepted in the outcome, i.e., one of the two coalitions accepting these two edges must be winning, meaning that either $C\in\mathcal{W}$ or $(\N\setminus\! C) \in \mathcal{W}$.

%At this point, the reader may wonder whether the monotonicity axiom really is required to obtain Theorem~\ref{thm:oligarchy}, or whether it may not be possible to instead derive closure under supersets using collective rationality. The following counterexample shows that monotonicity is required. Suppose we want to aggregate a profile of transitive graphs into a single transitive graph. Then the rule that accepts an edge if and only if either all agents accept it or only agent~1 does meets all the requirements of  Theorem~\ref{thm:oligarchy}, except for monotonicity, and it indeed is not a rule that conforms to our technical definition of an oligarchy (even if, conceptually, it of course is very similar).

Observe that the converse of Theorem~\ref{thm:dictator} holds as well: any dictatorship is unanimous, grounded, IIE, and collectively rational with respect to any graph property (and certainly with respect to those that are contagious, implicative, and disjunctive).\footnote{The same is not true for Theorem~\ref{thm:oligarchy}: it is not the case that every oligarchy is collectively rational with respect to every contagious and implicative graph property. The reason is that not every contagious and implicative graph property is closed under intersection, although many concrete such properties (e.g., transitivity) are. For example, the intersection rule does not preserve connectedness (which we have seen to be both contagious and implicative): if agent~1 provides the connected graph $\{(x,y),(x,z),(y,z)\}$ and agent~2 provides the connected graph $\{(x,y),(x,z),(z,y)\}$, then their intersection $\{(x,y),(x,z)\}$ nevertheless fails to be connected.} Thus, an alternative reading of Theorem~\ref{thm:dictator} is as that of a family of characterisation theorems of the dictatorships (with one characterisation for every $P$ that is contagious and implicative).

Our Theorem~\ref{thm:dictator} is related to generalisations of Arrow's Theorem to judgment aggregation~\cite{DietrichListSCW2007,DokowHolzmanJET2010}, particularly in the formulation due to \citet{DokowHolzmanJET2010}, who model sets of judgments (on $m$ issues) as binary vectors in some subspace of~$\{0,1\}^m$. It is possible to embed graph aggregation into this form of judgment aggregation, by adapting the well-known approach for embedding preference aggregation into judgment aggregation~\cite{DietrichListSCW2007,DokowHolzmanJET2010,GrandiEndrissAIJ2013}. This suggests that it should also be possible to derive Theorem~\ref{thm:dictator} as a special case of the main result of \citeauthor{DokowHolzmanJET2010}, which would involve showing that graph properties that are contagious, implicative, and disjunctive can be mapped into subspaces of $\{0,1\}^m$ (with $m=|\Edges|$) that, in the terminology of \citeauthor{DokowHolzmanJET2010}, are \emph{totally blocked} and \emph{not affine}.\footnote{Note that both \citet{DokowHolzmanJET2010} and \citet{DietrichListSCW2007} in fact prove \emph{characterisation results} (in a different sense of that word than we have used in Section~\ref{sec:axioms})  that have both an impossibility and a possibility component. To use our terminology, they formulate meta-properties that are such that, whenever they are met, then nondictatorial aggregation is impossible, while whenever they are not met, nondictatorial aggregation is possible. We do not consider this second direction here. The reason is that, rather than proving theorems of maximal logical strength, we are interested in theorems that are easy to apply. That this is the case for our choice of meta-properties will be demonstrated in Section~\ref{sec:applications}.} While we conjecture this to be possible in principle, we also conjecture any such proof to be at least as technically involved as our proof given here and certainly much less valuable from a ``didactic'' point of view. Indeed, our proof arguably is easier and clearer than both the proofs for the corresponding result in the more specific domain of preference aggregation (i.e., Arrow's Theorem)\footnote{This is true for proofs of Arrow's Theorem using the ultrafilter method, which is a refinement of the ``decisive coalition method'' going back to Arrow's original work~\cite{Arrow1963}. There are, however, other proofs available that exploit the specific structure of preferences, and thus do not generalise to, e.g., judgment aggregation, which some readers will find more accessible~\cite{GeanakoplosET2005}.} and the proofs for the corresponding results in the more general domain of judgment aggregation (i.e., the result of of \citet{DokowHolzmanJET2010} and its variant due to \citet{DietrichListSCW2007}). The reason is that our meta-properties encode directly what we require in the proof steps where they are used.

%%%%%%%%%%%%%%%%%%%%%%%%%%%%%%%%%%%%%%%%%%%%%%%%%%%%%%%%%%%%%%%%%%%%%%%%%%%%%%%%
\subsection{Variants and Instances of the General Impossibility Theorems}
\label{sec:variants}
%%%%%%%%%%%%%%%%%%%%%%%%%%%%%%%%%%%%%%%%%%%%%%%%%%%%%%%%%%%%%%%%%%%%%%%%%%%%%%%%

\noindent
In the remainder of this section, we shall briefly discuss the implications of of our general impossibility theorems for specific classes of graphs, particularly those that satisfy some of the properties of Table~\ref{tab:graph-properties}. We keep this discussion largely abstract; concrete applications will be discussed in Section~\ref{sec:applications}. But first let us consider a number of variants of our theorems and mention additional assumptions that would allow us to remove the technical constraint on nonreflexive edges in Theorems~\ref{thm:oligarchy} and~\ref{thm:dictator}, and to instead derive results on full dictatorships and full oligarchies, respectively.

First, note that if we remove the requirement of $P$ being contagious but add the assumption of $F$ being NR-neutral to Theorems~\ref{thm:oligarchy} and~\ref{thm:dictator}, we can still derive the same conclusions (namely, $F$ being NR-oligarchic or NR-dictatorial, respectively). If we impose full neutrality rather than just NR-neutrality, these conclusions can be strengthened to $F$ being fully oligarchic or dictatorial, respectively. For ease of reference, we state these variants here explicitly: 

\begin{theorem}\label{thm:neutral-oligarchy}
For $|V|\geq 3$, any unanimous, grounded, %monotonic, 
IIE, and (NR-)neutral  aggregation rule~$F$ that is collectively rational with respect to a graph property~$P$ that is implicative must be (NR-)oligarchic.
\end{theorem}

\begin{theorem}\label{thm:neutral-dictator}
For $|V|\geq 3$, any unanimous, grounded, IIE, and (NR-)neutral  aggregation rule~$F$ that is collectively rational with respect to a graph property~$P$ that is implicative and disjunctive must be (NR-)dictatorial.
\end{theorem}

\noindent 
As implicativeness and disjunctiveness are much less demanding properties than contagiousness and as neutrality is often a reasonable axiom to impose, these variants of our main theorems are of some practical interest.

Next, recall that by Proposition~\ref{prop:unanimity-reflexivity}, unanimity implies collective rationality with respect to reflexivity. Thus, our theorems remain true if we add reflexivity to the collective rationality requirements. In fact, they can be strengthened: for a unanimous rule and under the assumption that all input graphs are reflexive, every NR-dictatorial rule is in fact a full dictatorship and any NR-oligarchic rule is in fact a full oligarchy. Analogously, by Proposition~\ref{prop:groundedness-irreflexivity} and in view of our assumption of groundedness, we can alternatively add irreflexivity to the collective rationality requirements and strengthen our theorems in the same manner. Thus, we obtain two further variants each for each of Theorem~\ref{thm:dictator} and Theorem~\ref{thm:oligarchy}.

A simple instance of the first of these variants of Theorem~\ref{thm:dictator} is Arrow's Theorem for weak orders (i.e., binary relations that are reflexive, transitive, and complete). An aggregation rule mapping profiles of weak orders to weak orders, i.e., a \emph{social welfare function}~\cite{Arrow1963}, is simply a graph aggregation rule that is collectively rational with respect to reflexivity, transitivity, and completeness. Arrow uses two axioms, namely \emph{independence} (which is the same as our IIE axiom), and the \emph{weak Pareto condition}, according to which unanimously held strict preferences between two alternatives $x$ and $y$ should be respected by the aggregation rule. 

\begin{theorem}[\thmcite{Arrow1963}]\label{thm:arrow}
Any weakly Paretian and independent preference aggregation rule, mapping profiles of weak orders over three or more alternatives to weak orders, must be a dictatorship.
\end{theorem}

\begin{proof}
If we use the edges of a graph to represent weak preferences, then strict preference of $x$ over $y$ means that we accept edge $(x,y)$ but reject edge $(y,x)$. Thus, the weak Pareto condition together with IIE (independence) implies unanimity, while the weak Pareto condition together with collective rationality with respect to completeness implies groundedness.

Now the theorem follows immediately from Theorem~\ref{thm:dictator}, together with the insights 
that $(i)$~transitivity is a graph property that is contagious (Fact~\ref{fact:contagiousness}) and implicative (Fact~\ref{fact:implicativeness}),
that $(ii)$~completeness is a graph property that is disjunctive (Fact~\ref{fact:disjunctiveness}),
and that $(iii)$~reflexivity allows us to conclude that the aggregation rule must be a full dictatorship rather than just an NR-dictatorial rule.
\end{proof}

\noindent
Using the same approach, we can also easily derive a variant of Arrow's Theorem for strict linear preference orders (binary relations that are irreflexive, transitive, and complete) from Theorem~\ref{thm:dictator}. In this context, the weak Pareto condition is equivalent to the unanimity axiom, and groundedness is implied by the weak Pareto condition together with the collective rationality requirement for completeness.

\newcommand{\Yes}{$\checkmark$}
\newcommand{\No}{$\times$}
\begin{table}[t]
\centering
\begin{tabular}{lccc} \toprule
\textsc{Property} & \textsc{Contagious?} & \textsc{Implicative?} & \textsc{Disjunctive?} \\ \midrule
Reflexivity & \No & \No & \No \\
Irreflexivity & \No & \No & \No \\
Symmetry & \No & \No & \No \\
Antisymmetry & \No & \No & \No \\
Right Euclidean & \Yes & \Yes & \No \\
Left Euclidean  & \Yes & \Yes & \No \\
Transitivity & \Yes & \Yes & \No \\
Negative Transitivity & \Yes & \No & \Yes \\
Connectedness & \Yes & \Yes & \Yes \\
Completeness & \No & \No & \Yes \\
%Strong Completeness & \No & \No & \Yes \\
Nontriviality & \No & \No & \Yes \\
Seriality &  \No & \No & \Yes \\
\bottomrule
\end{tabular}
\caption{Meta-properties of common graph properties.\label{tab:meta-properties}}
\end{table}

But Arrow's Theorem now is just an example. We can immediately obtain any number of impossibility results such as this one, as long as the properties of the graphs we want to work with hit the appropriate meta-properties. Table~\ref{tab:meta-properties} summarises which of our standard graph properties are contagious (see Fact~\ref{fact:contagiousness}), implicative (see Fact~\ref{fact:implicativeness}), and disjunctive (see Fact~\ref{fact:disjunctiveness}), respectively. Any combination of graph properties that together hit all three graph properties, by Theorem~\ref{thm:dictator}, gives rise to an impossibility theorem saying that all relevant aggregation rules are NR-dictatorial.
% (and if we also assume either reflexivity or irreflexivity, they must be fully dictatorial). 
Similarly, any combination of graph properties that together hit the first two meta-properties, by Theorem~\ref{thm:oligarchy}, gives rise to an impossibility theorem saying that the only relevant aggregation rules are NR-oligarchic.
To be precise, when combining several graph properties, one needs to verify that the relevant richness conditions continue to be satisfied (which is trivially the case for all combinations of properties considered in Table~\ref{tab:meta-properties}).
To exemplify the possibilities, we state two concrete instances of our general results explicitly. They are particularly interesting, because they each require collective rationality with respect to just a single graph-property.

\begin{corollary}
For $|V|\geq 3$, any unanimous, grounded, %monotonic, 
and IIE aggregation rule that is collectively rational with respect to transitivity must be oligarchic on nonreflexive edges.
\end{corollary}

\begin{corollary}
For $|V|\geq 3$, any unanimous, grounded, and IIE aggregation rule that is collectively rational with respect to connectedness must be dictatorial on nonreflexive edges.
\end{corollary}

%%%%%%%%%%%%%%%%%%%%%%%%%%%%%%%%%%%%%%%%%%%%%%%%%%%%%%%%%%%%%%%%
\section{Integrity Constraints in Modal Logic}
\label{sec:modal}
%%%%%%%%%%%%%%%%%%%%%%%%%%%%%%%%%%%%%%%%%%%%%%%%%%%%%%%%%%%%%%%%

\noindent
So far we have worked with a definition of collective rationality that applies to every possible graph property (and we have specifically focused on common properties, such as transitivity). An alternative approach is to limit attention to properties that can be expressed in a restricted (logical) language. This is useful when we are interested in algorithmic aspects of collective rationality, e.g., the complexity of checking whether a given model satisfies the constraint (model checking). In our previous work on binary aggregation~\cite{GrandiEndrissAIJ2013}, we have focused on properties expressible in the language of propositional logic. Here, instead, we focus on fundamental properties of graphs that can be expressed using the language of \emph{modal logic}~\cite{BlackburnEtAl2001}. 

As we shall see, this is interesting not only because modal logic is a widely used language for describing graphs, but also because the standard semantics of modal logic suggests a new distinction of different \emph{levels} of collective rationality. After a brief review of relevant concepts from modal logic in Section~\ref{sec:modal-logic}, we introduce these three levels in Section~\ref{sec:levels}. One of them operates at the level of \emph{frames}, one at the level of \emph{models}, and one at the levels of possible \emph{worlds}. The first is equivalent to the basic notion of collective rationality used in the first part of this paper. Results for the other two are presented in Section~\ref{sec:model-CR} and~\ref{sec:world-CR}, respectively.

%%%%%%%%%%%%%%%%%%%%%%%%%%%%%%%%%%%%%%%%%%%%%%%%%%%%%%%%%%%%%%%%
\subsection{Background: Modal Logic}
\label{sec:modal-logic}
%%%%%%%%%%%%%%%%%%%%%%%%%%%%%%%%%%%%%%%%%%%%%%%%%%%%%%%%%%%%%%%%

\noindent
In what follows, we briefly review the basic concepts of modal logic and introduce the relevant notation~\cite{BlackburnEtAl2001}.
Fix a finite set $\Phi$ of propositional variables. The set of well-formed \emph{formulas}~$\phi$ is defined as follows (with $p$ ranging over the elements of~$\Phi$):
\begin{eqnarray*}
\phi & ::= & p \mid \neg\phi \mid \phi\wedge\phi \mid \phi\vee\phi \mid \phi\imp\phi \mid \Box\phi \mid \Diamond\phi
\end{eqnarray*}
A (Kripke) \emph{model} $M=\langle G,\val\,\rangle$ consists of a graph $G=\langle V,E\rangle$ and a \emph{valuation function} $\val:\Phi\to 2^V$. In line with standard terminology, we also refer to $G$ as a (Kripke) \emph{frame}, to $V$ as the set of \emph{possible worlds}, and to $E$ as an \emph{accessibility relation}. The valuation $\val$ is mapping propositional variables~$p$ to sets of possible worlds---the worlds where the $p$ in question is true.
The \emph{truth} of an arbitrary formula~$\phi$ at a world $x\in V$ in a model $M=\langle G,\val\,\rangle$, denoted $M,x\models\phi$, is defined recursively:
\begin{itemize}[noitemsep]
\item $M,x\models p$ if $x\in\val(p)$ for any $p\in\Phi$
\item $M,x\models \neg\phi$ if $M,x\not\models \phi$
\item $M,x\models \phi\wedge\psi$ if $M,x\models \phi$ and $M,x\models \psi$
%\item $M,x\models \phi\vee\psi$ if $M,x\models \phi$ or $M,x\models \psi$
%\item $M,x\models \Box\phi$ if $M,y\models \phi$ for all $y\in E(x)$
\item $M,x\models \Diamond\phi$ if $M,y\models \phi$ for some $y\in E(x)$
\end{itemize}

\noindent
Furthermore, $A\vee B$ is short for $\neg (\neg A\wedge \neg B)$, $A\imp B$ is short for $\neg(A\wedge \neg B)$, and $\Box A$ is short for $\neg\Diamond\neg A$.
Besides this notion of truth of $\phi$ at a given world, the semantics of modal logic provides two further ways of interpreting a formula~$\phi$ on a graph~$G$.
First, a formula $\phi$ is \emph{globally true} in model $M=\langle G,\val\,\rangle$, denoted $M\models\phi$, if $M,x\models\phi$ for every $x\in V$.
Second, $\phi$ is \emph{valid on frame} $G$, denoted $G\models\phi$, if $\langle G,\val\,\rangle\models\phi$ for every valuation~$\val$. 
Two formulas $\phi$ and $\psi$ are \emph{equivalent} if $M,x\models\phi$ implies $M,x\models\psi$ and \textit{vice versa}, for every model~$M$ and every world~$x$.

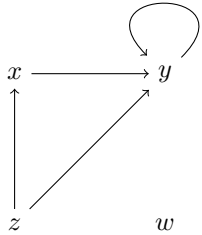
\begin{figure}%[t]
\centering
\begin{tikzpicture}
\node (x) at (0,0) {$x$};
\node (y) at (2,0) {$y$};
\node (z) at (0,-2) {$z$};
\node (w) at (2,-2) {$w$};
\draw [->] (x) edge (y); 
\draw [->] (z) edge (x); 
\draw [->] (z) edge (y); 
\draw [->] (y) edge[loop] (y); 
\end{tikzpicture}
\caption{Example for a modal logic frame with four possible worlds.\label{fig:modal-example}}
\end{figure}

\begin{example}[Frame validity and global truth]
Consider the frame~$G=\langle V,E\rangle$, with $V=\{x,y,z,w\}$, shown in Figure~\ref{fig:modal-example}. An example for a formula that is valid in this frame is $\Box q\imp\Box\Box q$, because---whatever the model---in every world for which all accessible worlds satisfy $q$ also all worlds accessible in exactly two steps satisfy~$q$. The formula $p\imp\Diamond p$, on the other hand, is not valid in $G$, because there exist models based on $G$, e.g., the model with $\val(p)=\{z\}$, in which it is not the case that from every world in which $p$ is true we can access some world that also satisfies~$p$. However, $p\imp\Diamond p$ is globally true in some models based on $G$, e.g., in the model with $\val(p)=\emptyset$.
\end{example}

\noindent
Recall that both truth at a world and global truth in a model are concepts that require the introduction of a valuation $\val$. Validity on a frame, on the other hand, is independent of the valuation and can be used to express global properties of frames, i.e., of graphs alone. For instance, it is well-known that $G=\langle V,E\rangle$ is reflexive (i.e., $E$ is a reflexive relation on $V$) if and only if the formula $p\imp \Diamond p$ is valid on $G$. Results of this kind belong to the realm of \emph{modal correspondence theory}~\cite{Benthem2001}. 
%directedness $\forall xyz.[xRy\wedge xRz \imp \exists s.(yRs\wedge zRs)]$, 
%$\Diamond\Box A\imp\Box\Diamond A$
%Goldblatt has $\forall xy.\exists z.(xRz \vee yRz)$. Blackburn et al.\ have a directedness condition involving two accessibility relations.
That is, through the concept of validity on a frame, we are able to express a property of a graph by means of a formula in modal logic. 

Some of the most fundamental frame properties considered in correspondence theory are listed in Table~\ref{tab:correspondences}. Such formulas can also be combined to characterise classes of graphs of interest. An equivalence relation, for instance, is a frame on which $p\rightarrow\Diamond p$, $p\rightarrow\Box\Diamond p$, and $\Diamond\Diamond p \rightarrow \Diamond p$ are valid. 
Note that not all graph properties have modal formulas defining them (e.g., irreflexivity, completeness, and negative transitivity do not).\footnote{The left-Euclidean property defined in Table~\ref{tab:graph-properties} also cannot be expressed directly. It corresponds to the formula for the right-Euclidean property interpreted on the inverse relation $E^{-1}$.}

%Let us briefly remark here that connectedness is a weak form of \emph{completeness}, which requires that any two vertices can be compared: $\forall xy.(xEy\vee yEx)$. Completeness plays a central role in social choice theory, as preference orders are usually assumed to be complete, but it is not a frame property that can be characterised by a modal formula.

\begin{table}%[t]
\centering
\begin{tabular}{ll}\toprule
\textsc{Property} & \textsc{Modal Formula}  \\ \midrule 
Reflexivity & $ p\imp \Diamond p$  \\
Symmetry & $p\imp\Box\Diamond p$ \\
Right Euclidean & $\Diamond p\imp\Box\Diamond p$   \\
Transitivity & $\Diamond\Diamond p\imp \Diamond p$  \\
Connectedness & $\Box(\Box p\imp q) \vee \Box(\Box q\imp p)$ \\ % Geach axiom
Seriality & $\Diamond(p\vee \neg p)$ \\
%Church-Rosser & $\Diamond\Box p\rightarrow \Box\Diamond p$ \\
\bottomrule
\end{tabular}
\caption{Common frame properties and the corresponding modal formulas.\label{tab:correspondences}}
\end{table}

%\ugnote{Should we add a paragraph or a section on model checking using modal logic?}
%\uenote{Not sure what you had in mind.}

%%%%%%%%%%%%%%%%%%%%%%%%%%%%%%%%%%%%%%%%%%%%%%%%%%%%%%%%%%%%%%%%%%%%%%%%%%%%%%%%
\subsection{Three Levels of Collective Rationality}
\label{sec:levels}
%%%%%%%%%%%%%%%%%%%%%%%%%%%%%%%%%%%%%%%%%%%%%%%%%%%%%%%%%%%%%%%%%%%%%%%%%%%%%%%%

\noindent
Given a set of propositional variables~$\Phi$, we shall refer to modal formulas $\phi$ constructed from $\Phi$ as \emph{modal integrity constraints}.
We now introduce three definitions of collective rationality with respect to a modal integrity constraint. What distinguishes them is the level (frame, model, world) at which the modal integrity constraint is interpreted. 

\begin{definition}%[F-CR]
\label{def:F-CR}
An aggregation rule $F$ is \textbf{frame collectively rational} with respect to a modal integrity constraint~$\phi$ if $\langle V,E_i\rangle\models\phi$ for all $i\in\N$ implies $\langle V,F(\prof{E})\rangle\models\phi$.
\end{definition}

\noindent
That is, $F$ is frame collectively rational with respect to~$\phi$ if validity of $\phi$ on all individual frames $\langle V,E_i\rangle$ implies validity of $\phi$ on the collective frame $\langle V,F(\prof{E})\rangle$. This is equivalent to our original Definition~\ref{def:CR}, with the only difference being that the property with respect to which we require collective rationality now has to be expressed by means of a modal formula.

\begin{definition}%[M-CR]
\label{def:M-CR}
An aggregation rule $F$ is \textbf{model collectively rational} with respect to a modal integrity constraint~$\phi$ if for every valuation $\val:\Phi\to 2^V$ we have $\langle\langle V,E_i\rangle,\val\,\rangle\models\phi$ for all $i\in\N$ implying $\langle\langle V,F(\prof{E})\rangle,\val\,\rangle\models\phi$.
\end{definition}

\noindent
That is, $F$ is model collectively rational with respect to $\phi$ if---for any valuation $\val$---global truth of $\phi$ in all individual models $\langle\langle V,E_i\rangle,\val\,\rangle$ implies global truth of $\phi$ in the collective model $\langle\langle V,F(\prof{E})\rangle,\val\,\rangle$.

\begin{definition}%[W-CR]
\label{def:W-CR}
An aggregation rule $F$ is \textbf{world collectively rational} with respect to a modal integrity constraint~$\phi$ if for every valuation $\val:\Phi\to 2^V$ and every world $x\in V$ we have $\langle\langle V,E_i\rangle,\val\,\rangle,x\models\phi$ for all $i\in\N$ implying $\langle\langle V,F(\prof{E})\rangle,\val\,\rangle,x\!\models\phi$.
\end{definition}

\noindent
Thus, $F$ is world collectively rational with respect to $\phi$ if---again, for any valuation---truth of $\phi$ at a given world in all individual models implies truth of $\phi$ at the same world in the collective model.

\begin{example}[Levels of collective rationality]\label{ex:modal-CR}
Let us go back to our Example~\ref{ex:CR}, in which aggregating three graphs that are serial by means of the majority rule yielded a fourth graph that fails to be serial. Specifically, in the majority graph the world~$w$ does not have a successor. In our discussion of Example~\ref{ex:CR}, we concluded that the majority rule is not collectively rational with respect to seriality, which in the terminology of Definition~\ref{def:F-CR} is expressed as the majority rule not being frame collectively rational with respect to $\Diamond (p\vee \neg p)$, a modal formula that corresponds to seriality. 
On the other hand, by Fact~\ref{fact:majority}, the majority rule is frame collectively rational with respect to $p\imp \Diamond p$, corresponding to reflexivity. %, because the individual graphs are not reflexive to begin with. 
But note that the majority rule is not model collectively rational with respect to the same formula $p\imp \Diamond p$. To see this, consider a model with a valuation $\val$ such that $p$ is true at every world. Then $p\imp \Diamond p$ is globally true in all individual models, but it is not globally true in the collective model, since the bottom world~$w$ is not connected to any of the other $p$-worlds (i.e., $p$ is true at $w$, but $\Diamond p$ is not).
%A rule that does preserve seriality is the simple successor-approval rule that accepts an edge if it is (tied for being) most often accepted amongst those with the same source. % (in this case, we obtain the union of the four individual graphs when we apply this rule).
\end{example}

%UG: commented the entire paragraph
%The three definitions can be summarised in the following way: 
%Individuals are assumed to submit an accessibility relation over a set of worlds. 
%If we have no information about the valuation function or the world that will be of interest for our application, then our focus will on properties that are valid on the individual frames (think of global properties of graphs, such as transitivity in preference orders), i.e., we will be interested in F-collectively rational. 
%If, instead, we have information about the valuation (for instance in the case of epistemic and doxastic applications in which traditionally the valuation is used to define the possible worlds), but we have no information about the particular world that will be of interest, then we will be interested in a weaker but still global form of collective rationality as given by M-collectively rational. 
%If we know the valuation function and we know that the focus will be on a given world (although we may not know which), then  we can use the local version of collective rationality given by W-collectively rational.
%F-collectively rational directly corresponds to our Definition~\ref{def:CR}, while M-collectively rational and W-collectively rational do not have close counterparts in other areas of social choice theory. That is, while the basic framework of graph aggregation may be considered a special case of binary aggregation, this correspondence does not extend to these novel notions of collective rationality.

\noindent
A straightforward analysis of Definitions~\ref{def:F-CR}--\ref{def:W-CR} yields the following result: 

\begin{proposition}\label{prop:CRimplications}
Let $F$ be an aggregation rule and let $\phi$ be a modal integrity constraint. Then the following implications hold:
\begin{enumerate}%[nolistsep]
\item[$(i)$] If $F$ is world collectively rational with respect to $\phi$, then $F$ is also model collectively rational with respect to~$\phi$.
\item[$(ii)$] If $F$ is model collectively rational with respect to $\phi$, then $F$ is also frame collectively rational with respect to~$\phi$.
\end{enumerate}
\end{proposition}

\noindent
These inclusions are strict. For example, the aggregation rule~$F$ that returns the full graph in case all individual graphs satisfy $\Diamond (p\vee \neg p)$, and the empty graph otherwise, is model collectively rational but not world collectively rational.  To see this, consider a profile of graphs with two worlds where $E_i=\{(x,y)\}$ for all $i\in\N$. The outcome returned by $F$ is the empty graph, in violation of world collective rationality with respect to $\Diamond (p\vee \neg p)$ at world~$x$. 
Moreover, Example~\ref{ex:modal-CR} can be used to show the strict implication in item~$(ii)$, since it concerns an aggregation rule that is frame collectively rational with respect to modal formula $p\rightarrow\Diamond p$ but not model collectively rational with respect to the same formula. 

Thus, frame collective rationality is the least demanding of our three notions of collective rationality and world collective rationality is the most demanding. Hence, negative results are strongest when formulated for frame collective rationality, while positive results are strongest when formulated for world collective rationality. Our (negative) impossibility results of Section~\ref{sec:impossibility} were indeed proved for frame collective rationality and these results thus immediately extend also to the other two levels (in those cases where the graph property in question has a corresponding modal formula). The (positive) possibility results for frame collective rationality of Section~\ref{sec:CR}, however, do not automatically transfer. Indeed, as we will see next, they cannot be extended even to the next level, namely that of model collective rationality. Following this, we will complete the picture by establishing a number of positive results for world collective rationality, which immediately transfer to the other two levels as well.

%%%%%%%%%%%%%%%%%%%%%%%%%%%%%%%%%%%%%%%%%%%%%%%%%%%%%%%%%%%%%%%%%%%%%%%%%%%%%%%%
\subsection{Limitative Results for Collective Rationality at the Level of Models} 
\label{sec:model-CR}
%%%%%%%%%%%%%%%%%%%%%%%%%%%%%%%%%%%%%%%%%%%%%%%%%%%%%%%%%%%%%%%%%%%%%%%%%%%%%%%%

\noindent
Recall that in Section~\ref{sec:CR} we have seen that every unanimous aggregation rule is collectively rational with respect to reflexivity (Proposition~\ref{prop:unanimity-reflexivity}) and every neutral aggregation rule is collectively rational with respect to symmetry (Proposition~\ref{prop:neutrality-symmetry}). Given the well-known results in modal correspondence theory for these two properties, which we recall in Table~\ref{tab:correspondences}, we can reformulate these results as follows:\footnote{Note that Proposition~\ref{prop:groundedness-irreflexivity} cannot be reformulated in an analogous manner, because irreflexivity cannot be expressed in terms of a modal formula.}
\begin{itemize}
\item Any unanimous aggregation rule is frame collectively rational with respect to $p\imp\Diamond p$.
\item Any neutral aggregation rule is frame collectively rational with respect to $p\imp\Box\Diamond p$.
\end{itemize}

%We begin with a simple positive result, showing how a basic aggregation axiom can guarantee the preservation of a simple graph property:
%\begin{proposition}\label{prop:XXXreflexivity}
%Every unanimous aggregator is F-collectively rational wrt.\ $p \imp \Diamond p$, corresponding to reflexivity.
%\end{proposition}
%\begin{proof}
%Let $F$ be a unanimous aggregator, and let $\prof{E}$ be a profile such that $p \imp \Diamond p$ is valid on all individual frames $\langle V,E_i\rangle$.  Recall from Section~\ref{sec:CR} that the validity of some formulas at the frame level corresponds to global properties of graphs, such as reflexivity. We can thus reason with the corresponding property rather than proving the validity of the modal formula directly. If $p \imp \Diamond p$ is valid on all individual frames, then all $E_i$ respect reflexivity, i.e., every individual graph includes all edges of the form $(x,x)$. It is then straightforward to conclude that, by unanimity of $F$, the same must be true for the collective graph, and thus that $p \imp \Diamond p$ is valid on the collective frame $\langle V,F(\prof{E})\rangle$.
%\end{proof}

\noindent
The following two examples show that these results are tight, in the sense that they cease to hold when we replace frame collective rationality by model collective rationality. Both examples use the intersection rule~$F_\cap$, which is both unanimous and neutral.
%The following example shows that Proposition~\ref{prop:XXXreflexivity} is tight, i.e., that unanimous aggregators need not be M-collectively rational (and thus also not W-collectively rational) wrt.\ reflexivity. Hence, Proposition~\ref{prop:XXXreflexivity} is the strongest possibility result that we can get amongst the three notions of collective rationality wrt.\ the modal integrity constraint corresponding to reflexivity.  

\begin{example}[Counterexample for $p\imp\Diamond p$]
Let $V=\{x,y\}$. 
Suppose two individuals provide the following two graphs: $E_1=\{(x,y),(y,y)\}$ and $E_2=\{(y,x),(x,x)\}$, i.e., $F_\cap$ will return the empty graph.
Now consider the three models we obtain for these three graphs when we use the valuation $\val(p)=\{x,y\}$, which makes $p$ true at every world. Then the formula $p\imp\Diamond p$ is globally true in the two individual models, but it is not globally true in the model based on the collective (empty) graph. Hence, the intersection rule, despite being unanimous, is not model collectively rational with respect to $p\imp\Diamond p$.
\end{example}

%Symmetry is more demanding a property than reflexivity, and unanimity alone does not suffice to preserve it. However, if we restrict ourselves to uniform quota rules, we obtain another possibility result:
%\begin{proposition}\label{prop:uniformquota-symmetry} 
%Every uniform quota rule is F-collectively rational wrt.\ $p\imp\Box\Diamond p$, corresponding to symmetry.
%\end{proposition}
%\begin{proof}
%Once more, since validity of the formula $p\imp\Box\Diamond p$ corresponds to the symmetry of the frame, we can reason with graph properties only. Suppose each individual graph respects symmetry. Then the number of individual graphs including edge $(x,y)$ will always equal to the number of individual graphs including $(y,x)$. Hence, either both or neither will meet the uniformly imposed quota.
%\end{proof}
%Note that uniformity is a necessary condition for Proposition~\ref{prop:uniformquota-symmetry} to hold (nonuniform quota rules do not preserve symmetry).
%As for the previous case, we now provide an example that shows that Proposition~\ref{prop:uniformquota-symmetry} ceases to hold when we replace F-collectively rational with the stronger M-collectively rational (and thus also W-collectively rational). 

\begin{example}[Counterexample for $p\imp\Box\Diamond p$]
Let $V=\{x,y,z\}$. Suppose two individuals report the graphs $E_1=\{(x,y),(y,z)\}$ and $E_2=\{(x,y),(y,x)\}$, respectively. If we aggregate using $F_\cap$, we obtain a collective graph with a single edge $(x,y)$. Now consider the valuation $\val(p)=\{x,z\}$. While the formula $p\imp\Box\Diamond p$ is globally true in both individual models, the same formula is not satisfied at $x$ in the collective model (while $p$ is true at $x$, $x$ is connected in the collective graph to $y$ at which $\Diamond p$ is not satisfied). Hence, despite being neutral, $F_\cap$ is not model collectively rational with respect to the modal formula corresponding to symmetry.
\end{example}

%On the other hand, quota rules that are not uniform need not preserve symmetry: if the quota associated with edge $(x,y)$ is higher than the quota associated with $(y,x)$, then we can easily construct a profile where the quota for $(y,x)$ is met but that for $(x,y)$ is not.

\noindent
For more demanding modal integrity constraints, the situation is even more bleak. For example, we have already seen that transitivity is not preserved under the majority rule, which is an aggregation rule that meets essentially all axioms of interest. This is precisely what the Condorcet paradox shows (see Example~\ref{ex:condorcet-paradox}). Thus, neither unanimity nor neutrality (nor any other basic axiom we have considered) could possibly guarantee an aggregation rule to be frame collectively rational, or indeed model collectively rational, with respect to $\Diamond\Diamond p \imp \Diamond p$, the modal formula corresponding to transitivity. The best we can say is that all oligarchic rules are frame collectively rational with respect to $\Diamond\Diamond p \imp \Diamond p$. This is so, because the intersection of several transitive graphs is always transitive itself. We conclude our discussion of limitative results with an example showing that even this basic result does not transfer to collective rationality at the level of models.

\begin{example}[Counterexample for $\Diamond\Diamond p \imp \Diamond p$]
Let $V=\{x,y,z,w\}$ and suppose two individuals submit the two graphs depicted to the left of the dashed line below:

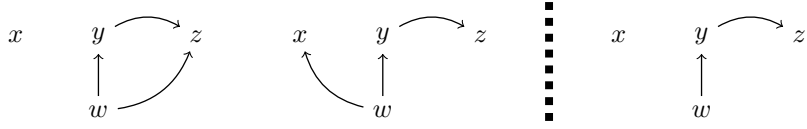
\begin{center}
\begin{tikzpicture}[auto]
  \node          (a)          {$x$};
  \node          (b) at (1.1,0) {$y$};
  \node          (c) at (2.4,0) {$z$};
  \node          (d) at (1.1,-1) {$w$};

  \path[->] (d) edge [->, line width=0.5pt]        node       {$$} (b)
 			 (b) edge [->, line width=0.5pt, bend left=30]        node       {$$} (c)
			 (d) edge [->, line width=0.5pt, bend right=30]        node       {$$} (c);
\end{tikzpicture}
\hspace*{7mm}
\begin{tikzpicture}[auto]
  \node          (a)          {$x$};
  \node          (b) at (1.1,0) {$y$};
  \node          (c) at (2.4,0) {$z$};
  \node          (d) at (1.1,-1) {$w$};

  \path[->] (d) edge [->, line width=0.5pt]        node       {$$} (b)
 			 (b) edge [->, line width=0.5pt, bend left=30]        node       {$$} (c)
			 (d) edge [->, line width=0.5pt, bend left=30]        node       {$$} (a);
\end{tikzpicture}
\hspace{4mm}
\begin{tikzpicture}
\draw [dashed] (0,1.6) edge[line width=3pt] (0,0);
\end{tikzpicture}
\hspace*{4mm}
\begin{tikzpicture}[auto]
  \node          (a)          {$x$};
  \node          (b) at (1.1,0) {$y$};
  \node          (c) at (2.4,0) {$z$};
  \node          (d) at (1.1,-1) {$w$};

  \path[->] (d) edge [->, line width=0.5pt]        node       {$$} (b)
 			 (b) edge [->, line width=0.5pt, bend left=30]        node       {$$} (c);
\end{tikzpicture}
\end{center}

\noindent
Under a valuation with $\val(p)=\{x,z,w\}$, the formula $\Diamond\Diamond p \imp \Diamond p$ is globally true in both models: world $w$ is the only world where $\Diamond\Diamond p$ is true, and in both models $w$ also satisfies $\Diamond p$. Now, the intersection rule will return the graph shown to the right of the dashed line.
%$E=\{(w,y),(y,z)\}$. 
In the corresponding model, the antecedent $\Diamond\Diamond p$ is still true at $w$, but $\Diamond p$ is not, since $p$ is false at $y$. Hence, the intersection rule is not model collectively rational with respect to $\Diamond\Diamond p\imp\Diamond p$.
\end{example}

%As noted in Section~\ref{sec:CR}, not all properties of graphs can be represented with modal formulas. Still, in line with the possibility results proven in this section, we can also study the preservation of graph properties that do not have a corresponding modal formula. Here is a simple example:

%\begin{proposition}\label{prop:irreflexivity}
%Every grounded aggregator preserves the property of irreflexivity: if all individual graphs are irreflexive, then so is the collective graph.
%\end{proposition}

%%%%%%%%%%%%%%%%%%%%%%%%%%%%%%%%%%%%%%%%%%%%%%%%%%%%%%%%%%%%%%%%%%%%%%%%%%%%%%%%
\subsection{Possibility Results for Collective Rationality at the Level of Worlds}
\label{sec:world-CR}
%%%%%%%%%%%%%%%%%%%%%%%%%%%%%%%%%%%%%%%%%%%%%%%%%%%%%%%%%%%%%%%%%%%%%%%%%%%%%%%%

\noindent
To complete the picture, we are now going to look for possibility results at the level of individual worlds. Recall that, by Proposition~\ref{prop:CRimplications}, any such result we are able to establish will immediately transfer to our other two notions of collective rationality as well.
%
%A \emph{safety result} establishes conditions under which an aggregation procedure is guaranteed to satisfy collective rationality with respect to certain properties (which we now model as modal integrity constraints). A safety result for the strong notion of world collectively rational implies that the same result holds also for the weaker notions of model- and frame collectively rational (just as an impossibility result for frame collectively rational implies the same impossibility also for model- and world collectively rational). 
%
Unlike in Section~\ref{sec:CR}, where we proved a number of simple possibility results for collective rationality at the level of frames for specific graph properties, the following results apply to all graph properties that can be expressed as modal integrity constraints meeting certain \emph{syntactic} restrictions. 
%The results we are able to provide are relatively weak. This is not surprising: they apply to the most demanding of our notions of collective rationality and they apply to an infinite range of graph properties.

Recall that a formula is said to be in \emph{negation normal form} (NNF) if it does not make use of the implication operator~$\imp$ and the negation operator~$\neg$ only occurs immediately in front of propositional variables. As is well known, any modal formula can be translated into an equivalent formula in NNF. We call a formula in NNF that does not have any occurrences of~$\Diamond$ a $\Box$-formula, and a formula in NNF without any occurrence of~$\Box$ a $\Diamond$-formula. 

The first straightforward observation to be made is that, if a formula~$\phi$ does not involve any modal operators ($\Box$ and $\Diamond$), then \emph{any} aggregation rule will be world collectively rational with respect to~$\phi$. This is immediate from Definition~\ref{def:W-CR}: the truth of such a $\phi$ only depends on the valuation~$\val$, which is not subject to change during aggregation.
For formulas involving only the universal modality~$\Box$, we need to ensure that the frame resulting from the aggregation does not include ``too many'' edges: 

\begin{proposition}\label{prop:box}
If an aggregation rule~$F$ is such that for every profile $\prof{E}$ there exists an individual $i^\star \in \N$ such that $F(\prof{E})\subseteq E_{i^\star}$, then $F$ is world collectively rational with respect to all $\Box$-formulas.
%An aggregator $F$ is W-collectively rational wrt.\ all $\Box$-formulas if and only if $F$ is grounded. FALSE
\end{proposition}

\begin{proof}
The proof hinges on a basic property of $\Box$-formulas, namely that of being preserved if the set of edges in a model gets reduced by deleting some of the edges.
So let $\phi$ be a $\Box$-formula and let $\prof{E}$ be a profile. Fix a world $x\in V$ and a valuation $\val$ such that $\langle \langle V,E_i\rangle,\val\,\rangle, x\models \phi$ for all $i\in \N$. 
In particular, we have $\langle\langle V,E_{i^\star}\rangle,\val\,\rangle, x\models \phi$. Since, by assumption, $F$ is such that  $F(\prof{E})\subseteq E_{i^\star}$, all boxed formulas that are true in $\langle\langle V,E_i\rangle,\val\,\rangle$ at $x$ are also true in the collective model $\langle\langle V,F(\prof{E})\rangle,\val\,\rangle$ at $x$; thus, $\langle\langle V,F(\prof{E})\rangle,\val\,\rangle,x\models\phi$.
\end{proof}

\noindent
Note that the individual $i^\star$ in Proposition~\ref{prop:box} need not be the same in all profiles. But of course, it \emph{can} be. This observation immediately leads to the following corollary:

\begin{corollary}\label{cor:oligarchy-box}
Any oligarchic aggregation rule is world collectively rational with respect to all $\Box$-formulas.
\end{corollary}

\noindent
For formulas involving only the existential modality~$\Diamond$, we have to ensure that the collective model includes ``enough'' edges:

\begin{proposition}\label{prop:diamond}
If an aggregation rule~$F$ is such that for every profile $\prof{E}$ there exists an individual $i^\star \in \N$ such that $F(\prof{E})\supseteq E_{i^\star}$, then $F$ is world collectively rational with respect to all $\Diamond$-formulas.
\end{proposition}

\begin{proof}
The proof is analogous to that of Proposition~\ref{prop:box}, this time using the property of $\Diamond$-formulas being preserved when the set of edges in a model is expanded by adding edges.
Let $\phi$ be a $\Diamond$-formula and let $\prof{E}$ be a profile such that, for a given world $x\in V$ and valuation $\val$, we have that $\langle \langle V,E_i\rangle,\val\,\rangle, x\models \phi$ for all $i\in \N$. 
By assumption, we know that that $F(\prof{E})\supseteq E_{i^\star}$. Hence, from the fact that $\langle\langle V,E_{i^\star}\rangle,\val\,\rangle, x\models \phi$ and that $\phi$ is a $\Diamond$-formula we can conclude that $\langle\langle V,F(\prof{E})\rangle,\val\,\rangle,x\models\phi$.
\end{proof}

\noindent
Examples for aggregation rules that satisfy the assumptions of Proposition~\ref{prop:diamond} are the dictatorships and the union rule. Oligarchic rules (other than the dictatorships), however, do not. Instead, in analogy to Corollary~\ref{cor:oligarchy-box}, any aggregation rule that always returns the union of the graphs provided by some fixed coalition is world collectively rational with respect to all $\Diamond$-formulas.

Propositions~\ref{prop:box} and~\ref{prop:diamond} together suggest a sufficient condition for an  aggregation rule to preserve truth for any kind of formula. Recall that a representative-voter rule is an aggregation rule~$F$ that is such that for every profile $\prof{E}$ there exists an individual $i^\star\in\N$ such that $F(\prof{E}) = E_{i^\star}$ (see Definition~\ref{def:representative-voter-rule}).

\begin{proposition}\label{prop:formulas}
%If an aggregation rule~$F$ is such that for every profile $\prof{E}$ there exists an individual $i^\star\in\N$ such that $F(\prof{E}) = E_{i^\star}$, then $F$ 
Any representative-voter rule is world collectively rational with respect to all modal integrity constraints.
\end{proposition}

\begin{proof}
Immediate from Definition~\ref{def:W-CR}: If the collective graph is a copy of one of the individual graphs, then all formulas that are true at the individual level will remain true at the collective level.
\end{proof}

\noindent 
Proposition~\ref{prop:formulas} is related to a result for binary aggregation characterising the representative-voter rules as those binary aggregation rules that are collectively rational with respect to all \emph{propositional} integrity constraints~\cite{GrandiEndrissAIJ2013}. 
%
%[[UG: this is one possibility to continue]] 
Interestingly, for graph aggregation and \emph{modal} integrity constraints, we do \emph{not} obtain such a result; the converse of Proposition~\ref{prop:formulas} does not hold. %, as shown by the following proposition:
The reason is that modal logic is not fully expressive on graphs: it cannot distinguish between models that are \emph{bisimilar}~\cite{BlackburnEtAl2001}. Rather than recalling the formal definition of bisimilarity here, we conclude by giving an example for an aggregation rule that, despite not being a representative-voter rule, is world collectively rational with respect to all modal integrity constraints.

\begin{example}[Aggregation and bisimilarity]
Let $V=\{x,y\}$ and $\Phi=\{p\}$. Suppose we are interested in world collective rationality at world~$x$ and relative to models with valuation~$\val(p)=\{x,y\}$. Let $F$ be the aggregation rule that is almost the dictatorship of agent~1, except that in case $E_1=\{(x,x),(y,y)\}$, rather than reproducing that graph, it returns the special graph $\{(x,y),(y,x)\}$. The two models based on these two graphs are bisimilar: no modal formula can distinguish them. For example, $\Diamond\Diamond p$ is true at both worlds in both of them, while $\neg p$ is false at both worlds in both of them. Clearly, $F$ is world collectively rational with respect to every modal integrity constraint, just as the dictatorship of agent~1 is. 
\end{example}

%%%%%%%%%%%%%%%%%%%%%%%%%%%%%%%%%%%%%%%%%%%%%%%%%%%%%%%%%%%%%%%%%%%%%%%%%%%%%%%%
\section{Applications in Artificial Intelligence}
\label{sec:applications}
%%%%%%%%%%%%%%%%%%%%%%%%%%%%%%%%%%%%%%%%%%%%%%%%%%%%%%%%%%%%%%%%%%%%%%%%%%%%%%%%

\noindent
In Section~\ref{sec:examples}, we introduced several scenarios that together exemplify the range of applications in which graph aggregation can play a role. In this section, we will revisit some of these scenarios, particularly those featuring prominently in AI research, and show how our results, notably our general impossibility theorems, can be put to use in these domains. Some of the results we will present are new, but most of them instead highlight how our approach can be used to clarify known results and to obtain significantly simpler proofs for them.
We will discuss applications of our approach 
to preference aggregation for agents that are not perfectly rational %and thus may have incomplete preferences 
(Section~\ref{sec:incomplete-preferences}), 
to nonmonotonic reasoning and belief merging (Section~\ref{sec:nonmono}),
to clustering analysis (Section~\ref{sec:clustering}),
and to abstract argumentation in multiagent systems (Section~\ref{sec:argumentation}).

Recall that an Arrovian aggregation rule is a rule that is unanimous, grounded, and IIE. We will use this terminology throughout this section. Also, to simplify the statements of theorems, when in this section we speak of ``aggregation rules for~$X$'', with $X$ being some family of graphs, we are referring to aggregation rules that are collectively rational with respect to the graph properties characterising~$X$. For example, Arrow's Theorem speaks about aggregation rules for weak orders, i.e., aggregation rules that are collectively rational with respect to the three graph properties defining weak orders.

%%%%%%%%%%%%%%%%%%%%%%%%%%%%%%%%%%%%%%%%%%%%%%%%%%%%%%%%%%%%%%%%%%%%%%%%%%%%%%%%
\subsection{Bounded Rationality: Aggregation of Incomplete Preferences}
\label{sec:incomplete-preferences}
%%%%%%%%%%%%%%%%%%%%%%%%%%%%%%%%%%%%%%%%%%%%%%%%%%%%%%%%%%%%%%%%%%%%%%%%%%%%%%%%

\noindent
In the economics literature, and thus in essentially all classical contributions to social choice theory, preferences are usually assumed to be \emph{complete}. Thus, for any two alternatives, a decision maker is assumed to be able to decide which of them she prefers or whether she is indifferent between them. In AI, on the other hand, such an assumption would often be considered controversial. Rather, an agent may not always be able to provide a complete preference order. This kind of bounded rationality could be due to the agent lacking relevant information or due to her lacking the necessary computational resources to arrive at a complete ranking. This is particularly relevant in domains where agents are asked to express preferences over very large sets of alternatives. Indeed, many of the formal preference representation languages developed in AI, such as CP-nets~\cite{BoutilierEtAlJAIR2004}, are not even able to express all complete preference orders~\cite{LangAMAI2004}. 

It therefore is important to understand the options available to us for aggregating \emph{incomplete preferences}, which are often modelled as \emph{preorders}, i.e., binary relations that are reflexive and transitive.\footnote{Thus, a weak order, which we have used to model preferences up to this point, is a preorder that is complete.} 
First, observe that Arrow's Theorem does \emph{not} apply to the aggregation of such incomplete preferences. A simple counterexample is the intersection rule, which is unanimous, grounded, IIE, and collectively rational with respect to both reflexivity and transitivity, i.e., it correctly maps profiles of preorders to single preorders---yet it is not a dictatorship. Of course, the intersection rule does not qualify as a very attractive rule either. It is an oligarchic rule, and in fact we can easily prove the following characterisation result:

\begin{theorem}
Let $F$ be an aggregation rule for preferences---modelled as preorders---over three or more alternatives.
Then $F$ is Arrovian %and monotonic 
if and only if it is oligarchic.  
\end{theorem}

%\uenote{Check relationship to Gibbard's oligarchy theorem (?). Also check whether Dietrich and List (``JA w/o full rationality'') have a theorem of this kind. UG: It is hard to find a decent statement of Gibbard's thm. The main point is that it does not use monotonicity. It is rather like Arrow's thm, but requires collectively rational wrt to weaker properties. But the input may be different than the output, individuals send preferences and the outcome can be a preorder (though I am not 100\% sure about this). It seems difficult to compare with our result for this last reason.}

\begin{proof}
The left-to-right direction follows from Theorem~\ref{thm:oligarchy}, as transitivity is contagious and implicative, and as reflexivity permits us to reduce NR-oligarchies to full oligarchies. 
The other direction is immediate.
\end{proof}

\noindent
Let us say that a preorder~$E$ \emph{has maxima} if there exists at least one element such that no other element is strictly preferred to it:
\[\exists x.\forall y.xEy \quad \mbox{($E$ has maxima)}\]
Thus, the preference order modelled by $E$ may be incomplete, but there is at least one element that is at least as preferable as any other. Similarly, let us say that $E$ \emph{has minima} if there exists at least one element that is at least as bad as any other element:
\[\exists x.\forall y.yEx \quad \mbox{($E$ has minima)}\]
\citet{PiniEtAlJLC2009} study Arrovian impossibilities for incomplete preferences in detail. They call an incomplete preference order (i.e., a preorder) ``restricted'' if it has maxima or minima (or both). Their main result is a variant of Arrow's Theorem for such restricted incomplete preferences~\cite[Theorem~5]{PiniEtAlJLC2009}:

\begin{theorem}[\thmcite{PiniEtAlJLC2009}]\label{thm:pini}
Any Arrovian aggregation rule for preferences---when modelled as preorders that have maxima or minima---over three or more alternatives must be a dictatorship.
\end{theorem}

\begin{proof}
The claim follows from Theorem~\ref{thm:dictator}, considering that transitivity is contagious and implicative, having maxima or minima is disjunctive, and reflexivity allows us to remove the restriction to nonreflexive edges.
In other words, the proof is identical to that of Theorem~\ref{thm:arrow}, except that now the disjunctive property of having maxima or minima takes over the role of the disjunctive property of completeness.
\end{proof}

\noindent
In fact, Theorem~\ref{thm:pini} is slightly stronger than the result stated by \citeauthor{PiniEtAlJLC2009}, who only require preferences to be restricted in the output but admit arbitrary preorders in the input (note that by admitting a wider range of inputs, encountering an impossibility becomes more likely). 
Besides making available a much simpler proof than the one originally given by \citeauthor{PiniEtAlJLC2009}, our approach shows that the focus on preorders that have maxima or minima is somewhat arbitrary. Any other property that is disjunctive, such as the strictly weaker nontriviality property (see Table~\ref{tab:graph-properties}), would have delivered the same result.

\citeauthor{PiniEtAlJLC2009} also prove variants of other classical theorems, notably the Muller-Satterthwaite Theorem and the Gibbard-Satterthwaite Theorem. Discussing these results is beyond  the scope of this paper. Having said this, it is well known that in the classical setting they can be obtained as relatively simple corollaries to Arrow's Theorem~\cite{EndrissLPT2011}, so our approach is likely to have fruitful applications also here.

%%%%%%%%%%%%%%%%%%%%%%%%%%%%%%%%%%%%%%%%%%%%%%%%%%%%%%%%%%%%%%%%%%%%%%%%%%%%%%%%
\subsection{Nonmonotonic Reasoning  and Belief Merging}
\label{sec:nonmono}
%%%%%%%%%%%%%%%%%%%%%%%%%%%%%%%%%%%%%%%%%%%%%%%%%%%%%%%%%%%%%%%%%%%%%%%%%%%%%%%%

\noindent
Aggregation plays a role in several contributions to the literature on nonmonotonic reasoning in AI. This is the case both for models of commonsense reasoning for a single intelligent agent who has to aggregate the possibly conflicting views arising from several different inference rules~\cite{DoyleWellmanAIJ1991}, and for work on merging the beliefs of several agents in a multiagent system~\cite{MaynardZhangLehmannJAIR2003}. In some approaches to nonmonotonic reasoning, alternative states of belief that an agent or a multiagent system might adopt are structured in terms of plausibility orderings that indicate which states are preferred to which other states according to a given criterion or a given individual agent. Such plausibility orders (often referred to as preferences in the literature) of course are graphs, so this boils down to a question of graph aggregation.\footnote{In other approaches to belief merging, belief bases themselves rather than the underlying plausibility orders are being aggregated~\cite{KoniecznyPinoPerezJLC2002}. These approaches are closely related to judgment aggregation~\cite{EveraereEtAlAAMAS2015}, rather than graph aggregation, and we shall not discuss them here.} Plausibility orders are reflexive and transitive, i.e., they are naturally modelled as preorders. In addition, different authors impose different additional requirements. We now review two contributions to nonmonotonic reasoning that involve graph aggregation.

The starting point of \citet{DoyleWellmanAIJ1991} is the observation that prior attempts at integrating various specialised patterns of commonsense inference into a universal logic of nonmonotonic reasoning have failed, and they try to explain this observation in terms of an Arrovian impossibility result for plausibility orders. They recognise that Arrow's Theorem does not extend to the aggregation of preorders, but also do not consider adding a completeness requirement as being appropriate in this context. Instead, besides independence and the weak Pareto condition, they invoke one additional axiom. \citeauthor{DoyleWellmanAIJ1991} call an aggregation rule~$F$ \emph{conflict-resolving} if, for all $x,y\in V$, it is the case that if $(x,y)\in E_i$ holds for at least one $i\in\N$, then $(x,y)\in F(\prof{E})$ or $(y,x)\in F(\prof{E})$ must hold as well. That is, if at least one agent ranks $x$ and $y$, then the output of $F$ must rank $x$ and $y$ as well (but not necessarily in the same direction). The main theorem of \citeauthor{DoyleWellmanAIJ1991} may be paraphrased as follows~\cite[Theorem~4.3]{DoyleWellmanAIJ1991}:

\begin{theorem}[\thmcite{DoyleWellmanAIJ1991}]
Any aggregation rule for plausibility orders---when modelled as preorders---over three or more states of belief that is Arrovian and conflict-resolving must be a dictatorship.
\end{theorem}

\begin{proof}
First, restrict attention to profiles where at least one agent submits a nonempty graph. But then the claim is strictly weaker than claiming that, for $|V|\geq 3$, any Arrovian aggregation rule for nontrivial preorders must be dictatorial, which follows from Theorem~\ref{thm:dictator} using the by now familiar approach.
Second, in case where every agent submits an empty graph, due to groundedness the collective graph will be empty as well, i.e., also in this case the dictator gets her will.
\end{proof}

\noindent
\citeauthor{DoyleWellmanAIJ1991} prove their result by inspection of a published proof of Arrow's Theorem, noting that, in that proof, collective rationality with respect to completeness is only ever used when at least one individual expresses a preference between the relevant two alternatives. This is a valid approach, and indeed, the result of \citeauthor{DoyleWellmanAIJ1991} is the theorem most similar to Arrow's original result amongst all the impossibility theorems discussed in this paper. Having said this, we believe that there is some added value in showing their result to be an immediate corollary to another theorem (as we have done here) rather than just showing how it follows \emph{from the proof} of another theorem (as \citeauthor{DoyleWellmanAIJ1991} have done), as this makes it considerably easier for others to verify the result and to prove similar new results themselves.

In work on belief merging, \citet{MaynardZhangLehmannJAIR2003} model plausibility orders as preorders that satisfy the property of negative transitivity (see Table~\ref{tab:graph-properties}), which they call \emph{modularity}. They argue that assuming negative transitivity rather than completeness, together with a modification of the independence axiom, allows them to circumvent Arrow's Theorem and to make reasonable aggregation rules available for belief merging. In the discussion of their result, they stress the significance of both of these changes. However, our analysis clearly shows that replacing completeness by negative transitivity alone has no effect on Arrow's impossibility, as negative transitivity is also a disjunctive property (see Fact~\ref{fact:disjunctiveness}). Hence, the crucial source for the possibility result of \citeauthor{MaynardZhangLehmannJAIR2003} must be their modification of the independence axiom. Indeed, this modification is rather substantial, as it allows for independence to be violated whenever not doing so would lead to what they term a ``conflict''. Thus, our approach is helpful also in this context in pinpointing the precise sources of impossibilities, thereby providing guidance on how they can be avoided.

%%%%%%%%%%%%%%%%%%%%%%%%%%%%%%%%%%%%%%%%%%%%%%%%%%%%%%%%%%%%%%%%%%%%%%%%%%%%%%%%
\subsection{Consensus Clustering}
\label{sec:clustering}
%%%%%%%%%%%%%%%%%%%%%%%%%%%%%%%%%%%%%%%%%%%%%%%%%%%%%%%%%%%%%%%%%%%%%%%%%%%%%%%%

\noindent
Given a set of data points, \emph{clustering} is the task of partitioning that set into subsets, in a way that in some sense is meaningful or useful~\cite{TanEtAl2005}. For example, someone designing an advertising campaign may wish to cluster a dataset about the past purchasing behaviour of a large group of people into a small number of groups of people with similar characteristics. Or someone designing a medical treatment may wish to cluster a medical dataset into subsets of patients with similar symptoms. Clustering has been exceptionally successful in practice, but is still lacking precise theoretical foundations. It is often difficult---and sometimes arguably impossible---to define what would constitute a ``correct'' clustering. The process of trying to find a compromise between the output of several different clustering algorithms is known as \emph{consensus clustering}. Consensus clustering can be modelled as a problem of graph aggregation. To see this, observe that a specific clustering of a given set of data points can be modelled as an \emph{equivalence relation} (i.e., a graph) on that set, by stipulating that two points are equivalent if and only of they belong to the same cluster. 

Recall that an equivalence relation on a set~$V$ is a binary relation on $V$ that is reflexive, symmetric, and transitive. To the best of our knowledge, \citet{Mirkin1975} was the first to analyse the aggregation of equivalence relations using the axiomatic method (see also \citet{BarthelemyEtAlJC1986}). Below we state a very similar result due to \citet{FishburnRubinsteinJC1986}, who in their paper refer to oligarchies as ``conjunctive operators''. 

\begin{theorem}[\thmcite{FishburnRubinsteinJC1986}]\label{thm:fishburn-rubinstein}
Any Arrovian aggregation rule for equivalence relations---which may represent alternative clusterings of a common dataset---over three or more data points must be an oligarchy. 
\end{theorem}

\begin{proof}
This follows from Theorem~\ref{thm:oligarchy}, together with the fact that transitivity is both a contagious and an implicative graph property, and the observation that collective rationality with respect to reflexivity eliminates the need to distinguish between NR-oligarchic and fully oligarchic rules. The additional requirement of collective rationality with respect to symmetry does not affect the result; in particular, it is easy to verify that the richness conditions in the definitions of contagiousness and implicativeness can still be met.
\end{proof}

\noindent
In fact, not every possible clustering will be useful. In particular, the clustering that puts every single data point in its own little cluster might meet most of the required definitions (e.g., it vacuously ensures that similarity between data points of the same cluster is always greater than similarity between data points belonging to different clusters), but it hardly will be helpful in understanding the structure of the data or in using it. Note that this kind of trivial clustering corresponds to the empty graph. Thus, we may assume that all individual graphs are nontrivial (as defined in Table~\ref{tab:graph-properties}) and we may wish to impose the same constraint on the result of the aggregation rule, i.e., we may wish to impose collective rationality with respect to nontriviality. If we do so, we can further tighten the impossibility result of \citeauthor{FishburnRubinsteinJC1986}:\footnote{We are grateful to Shai Ben-David for alerting us to this connection between consensus clustering and our Dictatorship Theorem (personal communication, June~2015).}

\begin{theorem}
Any Arrovian aggregation rule for nontrivial equivalence relations---which may represent alternative clusterings of a common dataset---over three or more data points must be a dictatorship. 
\end{theorem}

\begin{proof}
This follows from Theorem~\ref{thm:dictator}, in the same way as Theorem~\ref{thm:fishburn-rubinstein} follows from Theorem~\ref{thm:oligarchy}, together with the fact that nontriviality is a disjunctive graph property (see Fact~\ref{fact:disjunctiveness}).
\end{proof}

\noindent
Thus, it is impossible to design useful algorithms for consensus clustering that operate on each pair of data points independently.

While our approach applies to the problem of finding a consensus between the outputs produced by several clustering algorithms, we note that there also has been work on characterising those clustering algorithms themselves that is based on ideas originating in social choice theory~\cite{AckermanBenDavidNIPS2008,KleinbergNIPS2002}.

%%%%%%%%%%%%%%%%%%%%%%%%%%%%%%%%%%%%%%%%%%%%%%%%%%%%%%%%%%%%%%%%%%%%%%%%%%%%%%%%
\subsection{Multiagent Argumentation}
\label{sec:argumentation}
%%%%%%%%%%%%%%%%%%%%%%%%%%%%%%%%%%%%%%%%%%%%%%%%%%%%%%%%%%%%%%%%%%%%%%%%%%%%%%%%

%\ugnote{Write something about argumentation: merging Dung graphs \cite{CosteMarquisEtAlAIJ2007}, JA and argumentation [maybe..] \cite{CaminadaPigozzi2011,AwadEtAl2014,BoothEtAlKR2014}}
%\uenote{I don't think \cite{CaminadaPigozzi2011,AwadEtAl2014,BoothEtAlKR2014}  are specifically relevant, as they all seem to be closer to JA.}

\noindent
The final application scenario introduced in Section~\ref{sec:examples} we are going to discuss in some more detail here is that of argumentation in multiagent systems. An \emph{abstract argumentation framework} is a graph, the vertices of which are the \emph{arguments} and the edges of which represent a so-called \emph{attack-relation} between arguments. This model was introduced in the seminal work of \citet{DungAIJ1995}, who proposed several different semantics for abstract argumentation frameworks that specify principles according to which we may accept or reject arguments given the attacks between them. For example, if we accept argument~$x$, and if $x$ attacks $y$, then we should not also accept~$y$.

In a multiagent system, each agent may be associated with a different abstract argumentation framework on the same set of arguments, i.e., each agent may have different views on what constitutes a valid attack. We may then wish to merge these different frameworks to arrive at a suitable representation of the views of the group as a whole. The aggregation of abstract argumentation frameworks has been studied by a number of authors~\cite{CosteMarquisEtAlAIJ2007,TohmeEtAlFoIKS2008,DunneEtAlCOMMA2012,DelobelleEtAlIJCAI2015}.\footnote{In related work, other authors have studied the aggregation of alternative \emph{extensions} of a given common abstract argumentation framework, i.e., alternative choices on which arguments to accept~\cite{CaminadaPigozzi2011,RahwanTohmeAAMAS2010,BoothEtAlKR2014}. This line of work is more closely related to judgment aggregation and we shall not review it here. \citet{BodanzaEtAlECSQARU2009} compare these two distinct approaches of combining abstract argumentation and social choice theory.} 
%We also shall not review related work on the aggregation of \emph{weighted} argumentation frameworks~\cite{...}.
Next, we review some of this work and demonstrate that there are several interesting connections to our own work on graph aggregation, which suggests that graph aggregation can be fruitfully applied also in this domain.

\citet{CosteMarquisEtAlAIJ2007} were the first to consider the problem of aggregating several argumentation frameworks. They propose a distance-based method for aggregation. While they formulate the unanimity axiom as a relevant property in the context of aggregation of argumentation frameworks, they do not explicitly link their work to social choice theory.

\citet{TohmeEtAlFoIKS2008} were the first to make an explicit link to social choice theory. They formulate several choice-theoretic axioms for the aggregation of argumentation frameworks, e.g., an independence axiom and a (strong) monotonicity axiom (which is equivalent to the conjunction of our monotonicity axiom and IIE). They study collective rationality with respect to acyclicity. Acyclicity is an important graph property in the context of argumentation, because for an acyclic argumentation framework it is unambiguous which arguments to accept: accept all those that are not attacked by any argument or that are only attacked by arguments that themselves are attacked by some accepted argument. Acyclicity does not satisfy any of our three meta-properties (contagiousness, implicativeness, disjunctiveness), so our general impossibility theorems do not apply. Still, as \citeauthor{TohmeEtAlFoIKS2008} argue, the options for designing an aggregation rule that is collectively rational with respect to acyclicity are very limited. Clearly, every oligarchic rule is collectively rational with respect to acyclicity, because acyclicity of graphs is preserved under intersection. In addition, as \citeauthor{TohmeEtAlFoIKS2008} point out, also any aggregation rule based on a \emph{collegium}, i.e., a coalition of agents who each can veto any given edge from being accepted, but who may not be able to jointly enforce the acceptance of an edge (as would be the case in an oligarchy), is also collectively rational with respect to acyclicity, besides being Arrovian. % and monotonic.

\citet{DunneEtAlCOMMA2012} introduced further choice-theoretic axioms into the study of the aggregation of abstract argumentation frameworks (also discussed by \citet{DelobelleEtAlIJCAI2015}). Arguably, some of their ``axioms'' are better characterised as collective rationality requirements. For example, their ``nontriviality axiom'' in fact is just collective rationality with respect to nontriviality of graphs (as defined in Table~\ref{tab:graph-properties}). Probably the most important innovation in the work of \citet{DunneEtAlCOMMA2012} is the introduction of collective rationality requirements (albeit not under this name) with respect to graph properties that are specific to the context of abstract argumentation, such as the property of being ``decisive'' (in the sense of not permitting any ambiguity about which arguments are to be accepted). While, as explained above, acyclicity entails decisiveness, the converse is not true, i.e., studies of collective rationality with respect to acyclicity can only ever approximate the properties we should postulate for an aggregation rule for argumentation frameworks.

Modal logic can be used to define a semantics for argumentation frameworks by specifying rules for labelling arguments in a given argumentation framework as being either ``in'' or ``out'', or possibly ``undecided''~\cite{CaminadaGabbaySL2009,GrossiAAMAS2010}. This provides yet another connection to our work on graph aggregation. Let $\Phi=\{{\tt in},{\tt out},{\tt undec}\}$. We can use the following formula to express that every argument must get labelled using exactly one of these three options:
\[({\tt in} \wedge \neg{\tt out} \wedge \neg{\tt undec}) \vee (\neg{\tt in} \wedge {\tt out} \wedge \neg{\tt undec}) \vee (\neg{\tt in} \wedge \neg{\tt out} \wedge {\tt undec})\]
In addition, we can express constraints on the labelling of arguments that are linked to each other by means of the attack-relation. Let our graph describe the inverse of the attack-relation in an argumentation framework (rather than the attack-relation itself). Thus, the formula $\Diamond{\tt in}$, for example, will be true at a world, if that world represents an argument that is attacked by an argument that is ``in'', i.e., that is accepted. The formula $\Box{\tt out}$ is true if all attacking arguments are ``out'' (i.e., rejected). We now may wish to impose some of the following modal integrity constraints:
% see Def 5 in Caminada & Gabbay
\begin{itemize}
\item ${\tt in} \imp \Box{\tt out}$
(expressing that an argument can only be ``in'', if all of its attackers are ``out'')
\item $\Box{\tt out} \imp {\tt in}$
(expressing that, if all of an argument's attackers are ``out'', then it should be ``in'')
\item ${\tt out}\imp\Diamond{\tt in}$ 
(expressing that an argument should only be ``out'', if one of its attackers is ``in'')
\item $\Diamond{\tt in}\imp{\tt out}$
(expressing that an argument that has an attacker that is ``in'' must be ``out'')
\end{itemize}
A labelling that satisfies all of four of these constraints corresponds to what \citeauthor{DungAIJ1995} calls a \emph{complete extension}~\cite{DungAIJ1995,CaminadaGabbaySL2009}. A labelling that furthermore does not label any argument as being undecided, i.e., that makes $\neg{\tt undec}$ true at every world, corresponds to a so-called \emph{stable extension}~\cite{DungAIJ1995,CaminadaGabbaySL2009}.

Observe that each one of the four formulas above is equivalent to either a $\Box$-formula or a $\Diamond$-formula, although the conjunction of all four is not. Thus, in case we, for instance, are only interested in the first two of them, we can refer to Proposition~\ref{prop:box} to identify aggregation rules that are collectively rational with respect to these modal integrity constraints. If, however, we require an aggregation rule that preserves the property of having a complete (or stable) extension, then the best we can say at this point is that, by Proposition~\ref{prop:formulas}, any representative-voter rule meets this kind of requirement.

%%%%%%%%%%%%%%%%%%%%%%%%%%%%%%%%%%%%%%%%%%%%%%%%%%%%%%%%%%%%%%%%%%%%%%%%%%%%%%%%
\section{Conclusion}\label{sec:conclusion}       
%%%%%%%%%%%%%%%%%%%%%%%%%%%%%%%%%%%%%%%%%%%%%%%%%%%%%%%%%%%%%%%%%%%%%%%%%%%%%%%%

\noindent
We have introduced the problem of graph aggregation and analysed it in view of its possible use to combine information coming from different agents who each specify an alternative set of edges on the same set of vertices.
Our focus has been on the concept of collective rationality, i.e., the preservation of certain properties of graphs under aggregation. Our results are formulated with respect to various meta-properties that may or may not be met by a specific property of graphs one may be interested in. We have explored two different approaches to the definition of such meta-properties. Using a semantic approach, we have defined certain templates (namely contagiousness, implicativeness, and disjunctiveness), which are easy to recognise in common graph properties and which make the features of graph properties required to carry through our proofs particularly salient. Using a syntactic approach, we have used formulas expressible in certain fragments of modal logic to describe properties of graphs. 

Most of our technical results establish conditions under which it is either possible or impossible to guarantee collective rationality with respect to graph properties that meet certain meta-properties. Our main technical result is a generalisation of Arrow's Theorem for preference aggregation to aggregation problems for a large family of types of graphs that include the types of graphs used by Arrow to model preferences. To establish this theorem, as well as a closely related theorem identifying conditions that can only be satisfied by an oligarchic aggregation rule, we have refined the (ultra)filter method for proving impossibility theorems in social choice theory. Besides these technical contributions, we have also demonstrated how insights from the abstract setting of graph aggregation can be put to use in a variety of application domains.

While we have been able to demonstrate that our choice of meta-properties is particularly useful for quickly proving results in a wide variety of different domains, our impossibility theorems only establish sufficient conditions for impossibilities and there is room for future research on other such sufficient conditions and also for a complete characterisation of the family of types of graphs for which Arrovian aggregation is impossible. A good starting point for such an undertaking would be closely related work in judgment aggregation~\cite{DokowHolzmanJET2010,DietrichListSCW2007}. 

Besides such technical investigations, future work should continue to focus on applications of graph aggregation. Our discussion in Section~\ref{sec:applications} demonstrates the usefulness of adopting the general perspective of graph aggregation in the domains of preference aggregation, nonmonotonic reasoning and belief merging, cluster analysis, and argumentation. Future work should also address the other application scenarios identified in Section~\ref{sec:examples} and it should identify new ones. One promising direction concerns work on \emph{theory change} in the philosophy of science, where one recent model has used the Arrovian framework of preference aggregation to analyse how scientists choose between rival scientific theories in terms of preferences induced by criteria such as simplicity or fit with available data~\cite{OkashaMind2011}. The more general framework of graph aggregation opens up new possibilities for investigating the subtle differences that presumably exist between the preferences of an economic agent and the preferences induced by scientific criteria for accepting a novel theory. Another---entirely different but equally promising---direction for future research is in the area of the Semantic Web and concerns work on \emph{XML data integration}~\cite{HalevyEtAlVLDB2006}. The basic structure underlying documents encoded in XML (the \emph{extensible markup language}) is that of a tree, i.e., a special kind of graph. Thus, if we want to combine information encoded using XML that has been obtained from different sources on the Semantic Web, we need to use some form of graph aggregation as well.\footnote{This approach would be complementary to recently taken initial steps of another kind of use of the methodology of social choice theory in the area of the Semantic Web, namely the application of ideas originating in judgment aggregation to ontology merging~\cite{PorelloEndrissJLC2014}.} But also this extended list of potential applications is bound to be incomplete, given the ubiquity of graphs across so much of science and scholarship.

%Characterise common knowledge in terms of axioms?
%XML data aggregation
%\url{http://fox7.eu/wp-content/uploads/reasoningweb.pdf} or 
%\url{http://citeseerx.ist.psu.edu/viewdoc/download?doi=10.1.1.105.8830&rep=rep1&type=pdf})
%Aggregation of ontologies?}

%%%%%%%%%%%%%%%%%%%%%%%%%%%%%%%%%%%%%%%%%%%%%%%%%%%%%%%%%%%%%%%%%%%%%%%%%%%%%%%%
%\section*{References}
\bibliographystyle{abbrvnat}
\bibliography{ga}
%%%%%%%%%%%%%%%%%%%%%%%%%%%%%%%%%%%%%%%%%%%%%%%%%%%%%%%%%%%%%%%%%%%%%%%%%%%%%%%%

\end{document}